\documentclass{llncs}

\usepackage{amsfonts}
\usepackage{amssymb}
\usepackage{amsmath}
\usepackage{mathtools}
\usepackage{amstext}
\usepackage{array}
\usepackage{pifont}

\usepackage{enumitem}

\usepackage{misc}
\usepackage{xspace}
\usepackage{pgf}
\usepackage{tikz}
\usetikzlibrary{fit,calc}

\begin{document}

\title{Separating Positive and Negative Data Examples by Concepts and
  Formulas: \\ The Case of Restricted Signatures}

\author{
 Jean Christoph Jung$^{1}$
 \and
 Carsten Lutz$^{1}$ \and
 Hadrien Pulcini$^2$
\and
Frank Wolter$^2$
% \emails
% \{ clu\}@uni-bremen.de,
% wolter@liverpool.ac.uk
%
}

\institute{
 $^1$University of Bremen, Germany $\qquad$
 $^2$University of Liverpool, UK
}

\maketitle

\begin{abstract}
  We study the separation of positive and negative data examples in
terms of description logic (DL) concepts and formulas of decidable FO
fragments, in the presence of an ontology. In contrast to previous
work, we add a signature that specifies a subset of the symbols from
the data and ontology that can be used for separation. We consider
weak and strong versions of the resulting problem that differ in how
the negative examples are treated. Our main results are that (a
projective form of) the weak version is decidable in \ALCI while it is
undecidable in the guarded fragment GF, the guarded negation fragment
GNF, and the DL $\mathcal{ALCFIO}$, and that strong separability is
decidable in \ALCI, GF, and GNF. We also provide (mostly tight)
complexity bounds.
\end{abstract}

\newcommand\blfootnote[1]{%
  \begingroup
  \renewcommand\thefootnote{}\footnote{#1}%
  \addtocounter{footnote}{-1}%
  \endgroup
}

\enlargethispage*{5mm}
\section{Introduction}

% A stronger version of concept separability that is also considered in
% the literature requires that $\Kmc\models \neg C(a)$ for all $a\in N$,
% rather than only $\Kmc \not\models C(a)$.  Clearly, this is a
% meaningful notion only for DLs with negation.  We present some first
% results also for this version of concept separability, giving
% model-theoretic characterizations for expressive DLs based on
% bisimulations and proving that for $\Lmc \in \{ \ALC,\ALCI \}$, \Lmc
% concept separability becomes \ExpTime-complete in combined complexity
% and \coNPclass-complete in data complexity.

There are several applications that fall under the broad term of
supervised learning and seek to compute a logical expression that
separates positive from negative examples given in the form of labeled
data items in a knowledge base. A prominent example is concept
learning for description logics (DLs) where the aim is to support a
user in automatically constructing a concept description that can then
be used, for instance, in ontology engineering
\cite{DBLP:conf/ilp/BadeaN00,DBLP:journals/ml/LehmannH10,polconcept_learning,DBLP:journals/fuin/TranHHNN14,DBLP:conf/www/BuhmannLWB18,DBLP:conf/ekaw/Fanizzi0dE18,DBLP:conf/aaai/SarkerH19}. A
further example is reverse engineering of database queries (also
called query by example, QBE), which has a long history in database
research
\cite{DBLP:conf/sigmod/TranCP09,DBLP:journals/vldb/TranCP14,DBLP:conf/sigmod/ZhangEPS13,DBLP:conf/pods/WeissC17,kalashnikov2018fastqre,DBLP:journals/tods/ArenasD16,DBLP:conf/icdt/Barcelo017,DBLP:journals/tods/KimelfeldR18,martins2019reverse}
and which has also been studied in the presence of a DL ontology
\cite{GuJuSa-IJCAI18,DBLP:conf/gcai/Ortiz19}. Note that a closed world
semantics is adopted for QBE in databases while an open
world semantics is required when the data is assumed to be incomplete
as in the presence of ontologies, but also, for example, in reverse
engineering of SPARQL queries~\cite{DBLP:conf/www/ArenasDK16}. Another
example is entity
comparison in RDF graphs, where one aims to find meaningful
descriptions that separate one entity from
another~\cite{DBLP:conf/semweb/PetrovaSGH17,DBLP:conf/semweb/PetrovaKGH19}
and a
final example is generating referring expressions (GRE) where the aim
is to describe a single data item by a logical expression such as a DL
concept, separating it from all other data items. GRE has originated
in linguistics \cite{DBLP:journals/coling/KrahmerD12}, but has
recently received interest in DL-based ontology-mediated
querying \cite{DBLP:conf/kr/BorgidaTW16}.

A fundamental problem common to all these applications is to decide
whether a separating expression exists at all. There are several
degrees of freedom in defining this problem. One concerns the negative
examples: is it enough that they do not entail the separating formula
(\emph{weak separability}) or are they required to entail its negation
(\emph{strong separability})? Another one concerns the question
whether additional helper symbols are admitted in the separating
formula (\emph{projective separability}) or not (\emph{non-projective
  separability}). The emerging family of problems has recently been
investigated in \cite{DBLP:conf/ijcai/FunkJLPW19,KR}, concentrating on
the case where the separating expression is a DL concept or formulated
in a decidable fragment of first-order logic (FO) such as the guarded
fragment (GF) and the guarded negation fragment (GNF).

In this paper, we add a signature
$\Sigma$ % (set of
% concept and role names)
that is given as an additional input and require separating
expressions to be formulated in~$\Sigma$ (in the non-projective
case). This makes it possible to `direct' separation towards
expressions based on desired features and to exclude features that are
not supposed to be used for separation such as gender and skin
color. In the projective case, helper symbols from outside of $\Sigma$
are also admitted, but must be `fresh' in that they cannot occur in
the given knowledge base. Argueably, such fresh symbols make the
constructed separating expressions less intuitive from an application
perspective and more difficult to understand. However, they sometimes
increase the separating power and they emerge naturally from a
technical
perspective.

The signature $\Sigma$ brings the separation problem closer to the
problem of deciding whether an ontology is a conservative extension of
another ontology \cite{JLMSW17}, also a form of separation,
and to deciding the existence of uniform interpolants
\cite{DBLP:conf/ijcai/LutzW11}.  It turns out, in fact, that lower
bounds for these problems can often be adapted to weak
separability with signature. In constrast, for strong separability we observe a
close connection to Craig interpolation.

We consider both weak and strong separability, generally assuming that
the ontology is formulated in the same logic that is used for
separation. We concentrate on combined complexity, that is, the input
to the decision problems consists of the knowledge base that comprises
an ABox and an ontology, the positive and negative examples in the
form of lists of individuals (for DLs) or lists of tuples of
individuals (for FO fragments that support more than one free
variable), and the signature. In the following, we summarize our main results.

We start with weak projective separability in \ALCI, present a
characterization in terms of $\Sigma$-homomorphisms that generalizes
characterizations from \cite{DBLP:conf/ijcai/FunkJLPW19,KR}, and then
give a decision procedure based on tree automata. This yields a
\TwoExpTime upper bound, and a matching lower bound is obtained by
reduction from conservative extensions. In contrast, weak projective
(and non-projective) separability in \ALCI without a signature is only \NExpTime-complete
\cite{DBLP:conf/ijcai/FunkJLPW19}. The non-projective case with
signature remains
open. We then show that weak separability is undecidable in any
fragment of FO that extends GF (such as GNF) or $\mathcal{ALCFIO}$ (such as the two-variable
fragment with
counting, $\text{C}^2$). In both cases,
the proof is by adaptation of undecidability proofs for conservative
extensions, from \cite{JLMSW17} and \cite{DBLP:conf/kr/GhilardiLW06}
respectively, and applies to both the projective and the
non-projective case. This should be contrasted with the fact that weak
separability is decidable and \TwoExpTime-complete for GF and for GNF
without a signature, both in the projective and in the non-projective
case~\cite{KR}. The decidability status of (any version of)
separability in $\mathcal{ALCFIO}$ without a signature is open. It is
known, however, that projective and non-projective weak separability
without a signature are undecidable in the two-variable fragment
FO$^2$ of FO~\cite{KR}.

We then turn to strong separability. Here, the projective and the
non-pro\-jective case coincide and will thus not be distinguished in
what follows. We again start with \ALCI for which we show
\TwoExpTime-completeness. The proofs, however, are rather different
than in the weak case. For the upper bound, we characterize
non-separability in terms of the existence of a set of types that are
amalgamable in the sense that they can be realized in a model of the
ontology at elements that are all $\ALCI(\Sigma)$-bisimilar, and that
satisfy certain additional properties. To identify sets of amalgamable
types, we use an approach that is loosely in the style of type
elimination procedures. A matching lower bound is proved by a
reduction from the word problem of exponentially space bounded ATMs.
We remark that in the strong case, the increase in complexity that
results from adding a signature is even more pronounced. In fact,
strong separability without a signature is only \ExpTime-complete in
\ALCI~\cite{KR}. We then turn to GF and GNF and establish a close link
between strong separability and interpolant existence, the problem to
decide for formulas $\varphi,\psi$ in a language $\Lmc$ whether there
exists a formula $\chi$ in $\Lmc$ using only the shared symbols of
$\varphi$ and $\psi$ such that both $\varphi\rightarrow \chi$ and
$\chi\rightarrow \psi$ are valid. We show that strong separability
with signature in GF and GNF are polynomial time reducible to
interpolant existence in GF and GNF, respectively. GNF enjoys the
Craig interpolation property (CIP), that is, there is such a
formula $\chi$ whenever
$\varphi\rightarrow\psi$ is valid. Thus, from the CIP
of GNF and the fact that validity in GNF is
\TwoExpTime-complete~\cite{DBLP:journals/jsyml/BaranyBC18}, we obtain
that strong separability with signature in GNF is
\TwoExpTime-complete. GF fails to have the CIP and
\ThreeExpTime-completeness for interpolant existence has only recently been
established~\cite{jung2020living}. We thus obtain a \ThreeExpTime upper bound for
strong separability with signature in GF. A matching lower bound can
be shown similar to the proof of \ThreeExpTime-hardness for interpolant
existence.
We note that strong separability without signature is \TwoExpTime-complete in both GNF and GF~\cite{KR}.

\section{Preliminaries}
  
Let $\Sigma_{\text{full}}$ be a set of \emph{relation symbols} 
that contains countably many symbols of every arity $n\geq 1$ and let
$\text{Const}$ be a countably infinite set of \emph{constants}.  A
\emph{signature} is a set of relation symbols $\Sigma \subseteq
\Sigma_{\text{full}}$.  We write $\vec{a}$ for a tuple
$(a_{1},\ldots,a_{n})$ of constants. A \emph{database} $\Dmc$ is a
finite set of \emph{ground atoms} $R(\vec{a})$, where $R\in
\Sigma_{\text{full}}$ has arity $n$ and $\vec{a}$ is a tuple of
constants from Const of length $n$. We use $\text{cons}(\Dmc)$ 
to denote the set of constants that occur in $\Dmc$.

Denote by FO the set of first-order (FO) formulas constructed from
constant-free atomic formulas $x=y$ and $R(\vec{x})$,
$R\in \Sigma_{\text{full}}$, 
using conjunction, disjunction, negation, and existential and
universal quantification.
% Note that formulas in FO do not use constants. 
As usual, we write $\vp(\vec{x})$ to indicate that the free variables
in FO-formula $\vp$ are all from $\vec{x}$ and call a formula
\emph{open} if it has at least one free variable and a \emph{sentence}
otherwise. Note that we do not admit constants in FO-formulas. 

An ontology \Omc is a finite set of FO-sentences, and a
\emph{knowledge base (KB)} is a pair $\Kmc=(\Omc,\Dmc)$ of an ontology
\Omc and a database \Dmc. As usual, KBs $\Kmc=(\Omc,\Dmc)$ are
interpreted in \emph{relational structures}
$ \Amf=(\text{dom}(\Amf),(R^{\Amf})_{R\in
  \Sigma_{\text{full}}},(c^{\Amf})_{c \in \text{Const}}) $
where $\text{dom}(\Amf)$ is the non-empty \emph{domain} of $\Amf$,
each $R^{\Amf}$ is a relation over $\text{dom}(\Amf)$ whose arity
matches that of $R$, and $c^{\Amf} \in \text{dom}(\Amf)$ for all $c\in
\text{Const}$. Note that we do not make the \emph{unique name
assumption (UNA)}, that is $c_1^\Amf=c_2^\Amf$ might hold even when
$c_1 \neq c_2$. This is in fact essential for several of our
results.
%If we want to emphasize that $\Amf$ interprets exactly the symbols
%from signature $\Sigma \subseteq \Sigma_\text{full}$, we speak of a
%$\Sigma$-structure. {\color{blue}needed?} 
A structure $\Amf$ is a \emph{model of a KB} $\Kmc=(\Omc,\Dmc)$ if it
satisfies all sentences in $\Omc$ and all ground atoms in $\Dmc$. A KB
$\Kmc$ is \emph{satisfiable} if there exists a model of~$\Kmc$.

We introduce two fragments of FO, the guarded fragment and the
description logic \ALCI.  In the \emph{guarded fragment (GF)} of
FO~\cite{ANvB98,DBLP:journals/jsyml/Gradel99}, formulas are built from
atomic formulas $R(\vec{x})$ and $x=y$ by applying the Boolean
connectives and \emph{guarded quantifiers} of the form $$ \forall
\vec{y}(\alpha(\vec{x},\vec{y})\rightarrow \varphi(\vec{x},\vec{y}))
\text{ and } \exists \vec{y}(\alpha(\vec{x},\vec{y})\wedge
\varphi(\vec{x},\vec{y})) $$ where $\varphi(\vec{x},\vec{y})$ is a
guarded formula % with free
% variables among $\vec{x},\vec{y}$
and $\alpha(\vec{x},\vec{y})$ is an atomic formula or an equality
$x=y$ that contains all variables in $[\vec{x}] \cup [\vec{y}]$. The
formula $\alpha$ is called the \emph{guard} of the quantifier. We say
that an ontology \Omc is a \emph{GF-ontology} if all formulas in \Omc
are from GF, and likewise for knowledge bases.

We next introduce the DL \ALCI. In this context, unary relation
symbols are called \emph{concept names} and binary relation
symbols are called
\emph{role names}~\cite{handbook,DL-Textbook}. 
A \emph{role} is a role name or an \emph{inverse role} $R^-$ with $R$
a role name. For uniformity, we set $(R^{-})^{-}=R$. 
\emph{$\mathcal{ALCI}$-concepts} are defined by the
grammar
$$
C,D ~::=~ A \mid %\top \mid 
\neg C \mid %C \sqcup D \mid 
C \sqcap D \mid \exists R.C % \mid \forall R.C
$$
where $A$ ranges over concept names and $R$ over roles.  As usual, we
write $\top$ to abbreviate $A \sqcup \neg A$ for some fixed concept
name $A$, $\top$ for $\neg \bot$, $C \sqcup D$ for $\neg (\neg C
\sqcap \neg D)$, $C \rightarrow D$ for $\neg C \sqcup D$, and $\forall
R . C$ for $\neg \exists R. \neg C$.
An $\mathcal{ALCI}$-\emph{concept inclusion (CI)} takes the form
$C\sqsubseteq D$ where $C$ and $D$ are $\mathcal{ALCI}$-concepts. An
$\mathcal{ALCI}$-\emph{ontology} is a finite set of
$\mathcal{ALCI}$-CIs. An \emph{$\ALCI$-KB} $\Kmc=(\Omc,\Dmc)$ consists
of an $\ALCI$-ontology $\Omc$ and a database \Dmc. Here and
in general in the context of \ALCI, we assume that databases use only unary
and binary relation symbols. We sometimes also mention the fragment
\ALC of \ALCI in which inverse roles are not
available. % If not relevant or
% understood, we may drop $\Sigma$ and the name of the DL and for
% example speak about $\mathcal{ALCI}$-concepts and simply about
% concepts.

To obtain a semantics, every \ALCI-concept
$C$ can be translated into an FO-formula $C^\dag$ with one free variable $x$:
$$
\begin{array}{rcl}
  A^{\dag}  &=& A(x) \\
  (C \sqcap D)^\dag &=& C^\dag \sqcap D^\dag \\
 (\exists R.C)^\dag &=& \exists y \, (R(x,y) \land C^\dag[y/x]) \\
 (\exists R^-.C)^\dag &=& \exists y \, (R(y,x) \land C^\dag[y/x]). 
\end{array}
$$
%
% \begin{align*}
% & A^{\dag}  = A(x), \quad  \top^{\dag} =  (x=x), \\ % \quad \bot^{\dag} = \neg(x=x), \\
% & ^\dag\text{commutes with the Booleans (changing $\sqcap$ to $\land$ and $\sqcup$ to $\lor$),}\\ 
% % &\  (\neg C)^\dag = \neg C^\dag, (C \sqcup D)^\dag = C^\dag \lor D^\dag, \ (C \sqcap D)^\dag = C^\dag \land D^\dag,\\ 
% %& (C \to D)^\dag = C^\dag \to D^\dag,\ 
% & (\exists R.C)^\dag = \exists y \, (R(x,y) \land C^\dag[y/x]),\\ 
% & (\forall R.C)^\dag = \forall y \, (R(x,y) \to C^\dag[y/x]).
% \end{align*}
%
The \emph{extension} $C^{\Amf}$ of a concept $C$ in a structure
$\Amf$ is \mbox{defined as}
$
C^{\Amf} = \{ a\in \text{dom}(\Amf) \mid \Amf\models C^{\dag}(a)\}.
$
A CI $C\sqsubseteq D$ is regarded as a shorthand for the
FO-sentence $\forall x \, (C^\dag(x) \to D^\dag(x))$. Thus, every
\ALCI-concept can be viewed as a GF-formula and every \ALCI-ontology
can be viewed as a GF-ontology. By economically reusing variables, we
can even obtain formulas and ontologies from $\text{GF} \cap \text{FO}^2$.
We write $\Omc\models C \sqsubseteq D$
to say that CI $C\sqsubseteq D$ is a consequence of ontology $\Omc$,
that is, $C^\Amf \subseteq D^\Amf$ holds in every model $\Amf$ of
$\Omc$.  Concepts $C$ and $D$ are \emph{equivalent} w.r.t.\ an
ontology \Omc if $\Omc \models C\sqsubseteq D$ and $\Omc \models
D\sqsubseteq C$. 

The \emph{Gaifman graph} $G_{\Amf}$ of a structure $\Amf$ is the
undirected graph with set
of vertices $\text{dom}(\Amf)$ and an edge $\{ d,e \}$ whenever there
exists $\vec{a}\in R^{\Amf}$ that contains $d,e$ for some relation $R$.
The \emph{distance} $\text{dist}_{\Amf}(a,b)$ between
$a,b\in \text{dom}(\Amf)$ is defined as the length of a shortest path
from $a$ to $b$, if such a path exists.  Otherwise
$\text{dist}_{\Amf}(a,b)=\infty$.  The \emph{maximal connected
  component (mcc)} $\Amf_{\text{con}(a)}$ of $a$ in $\Amf$ is the
substructure of $\Amf$ induced by the set of all $b$ such that
$\text{dist}_{\Amf}(a,b)<\infty$.

Let $\Amf$ be a structure such that $R^{\Amf}=\emptyset$ for any
relation symbol $R$ of arity $>2$. We say that $\Amf$ is
\emph{tree-shaped} if $G_{\Amf}$ is a tree without reflexive loops and
$R^{\Amf} \cap S^{\Amf} = \emptyset$ for all distinct roles $R,S$. We
say that $\Amf$ has \emph{finite outdegree} if $G_{\Amf}$ has finite
outdegree. A structure $\Amf$ is a \emph{forest structure} w.r.t.\ an
$\mathcal{ALCI}$-KB $\Kmc=(\Omc,\Dmc)$ if the undirected graph $(V,E)$
with
\begin{align*}
V & = \text{dom}(\Amf) \\
E & = \{ \{d,e\}\mid (d,e)\in R^\Amf\text{ for some $R$}\}\setminus
\{ \{d,e\}\mid d,e\in \text{cons}(\Dmc)\}
\end{align*}
is a tree. We drop `w.r.t.\ \Kmc' if \Kmc is clear from the context
and speak of a \emph{forest model} of \Kmc when \Amf is a model of
\Kmc.
%
%  A forest model of \Kmc has finite outdegree if 
% $(V,E)$ has finite outdegree.
 The following result is well known.
\begin{lemma}\label{lem:forestmodelcompleteness}
	Let $\Kmc$ be an $\mathcal{ALCI}$-KB and $C$ an $\mathcal{ALCI}$-concept.
	If $\Kmc\not\models C(a)$, then there exists a forest model $\Amf$ of $\Kmc$
	of finite outdegree with $a\not\in C^{\Amf}$.
\end{lemma}

We close this section with introducing homomorphisms and bisimulations.
Let $\Sigma$ be a signature. A \emph{$\Sigma$-homomorphism} $h$ from a
structure $\Amf$ to a structure $\Bmf$ is a function
$h:\text{dom}(\Amf)
\rightarrow \text{dom}(\Bmf)$ such that $\vec{a}\in
R^{\Amf}$ implies $h(\vec{a})\in R^{\Bmf}$ for all relation symbols
$R\in \Sigma$ and tuples $\vec{a}$ and with $h(\vec{a})$ being defined
component wise in the expected way. 
%clu:this is implied
%of the same length as the arity of $R$. 
Note that homomorphisms need not preserve 
constant symbols.  Every database $\Dmc$ gives rise to the finite
structure $\Amf_{\Dmc}$ with $\text{dom}(\Amf_\Dmc)=\text{cons}(\Dmc)$
and $\vec{a}\in R^{\Amf_{\Dmc}}$ iff $R(\vec{a})\in \Dmc$. A
$\Sigma$-homomorphism from database $\Dmc$ to structure $\Amf$ is a
$\Sigma$-homomorphism from $\Amf_{\Dmc}$ to $\Amf$.  A \emph{pointed structure}
takes the form $\Amf,\vec{a}$ with \Amf a structure and $\vec{a}$ a
tuple of elements of $\text{dom}(\Amf)$.  
%
% For a structure \Amf and
% tuple of constants $\vec{c}=(c_{1},\ldots,c_{n})$, we set
% $\vec{c}^{\Amf}= (c_{1}^{\Amf},\ldots,c_{n}^{\Amf})$. 
A homomorphism
from pointed structure $\Amf,\vec{a}$ to pointed structure
$\Bmf,\vec{b}$ is a homomorphism $h$ from $\Amf$ to $\Bmf$ with
$h(\vec{a})=\vec{b}$.  We write
$\Amf,\vec{a} \rightarrow \Bmf,\vec{b}$ to indicate the existence of
such a homomorphism.

We introduce $\ALCI(\Sigma)$-bisimulations between structures $\Amf$ and $\Bmf$ that 
interpret relations of arity at most two. Let $\Sigma$ be a signature. A
relation $S \subseteq \text{dom}(\Amf) \times \text{dom}(\Bmf)$ is an
\emph{$\mathcal{ALCI}(\Sigma)$-bisimulation between $\Amf$ and $\Bmf$} if 
the following conditions hold:
\begin{enumerate}
	\item for all $(d,e)\in S$:
	$d \in A^{\Amf}$ iff $e\in A^{\Bmf}$; 
	\item if $(d,e)\in S$ and $(d,d')\in R^{\Amf}$, 
	then there is a $e'$ with $(e,e')\in R^{\Bmf}$ and
	$(d',e')\in S$;
	\item if $(d,e)\in S$ and $(e,e')\in R^{\Bmf}$, 
	then there is a $d'$ with $(d,d')\in R^{\Amf}$ and
	$(d',e')\in S$,
\end{enumerate}	
where $A$ ranges over all concept
names in $\Sigma$ and $R$ over all $\Sigma$-roles.
We write $\Amf,d
\sim_{\mathcal{ALCI},\Sigma}\Bmf,e$ and call $\Amf,d$ and $\Bmf,e$
\emph{$\mathcal{ALCI}(\Sigma)$-bisimilar} if there exists an
$\mathcal{ALCI}(\Sigma)$-bisimulation $S$ such that
$(d,e)\in S$. 

The next lemma explains why $\mathcal{ALCI}(\Sigma)$-bisimulations are
relevant~\cite{TBoxpaper,goranko20075}. 
We say that $\Amf,d$ and $\Bmf,e$ are \emph{$\ALCI(\Sigma)$-equivalent},
in symbols $\Amf,d\equiv_{\ALCI,\Sigma}\Bmf,e$ if $d\in C^{\Amf}$ iff
$e\in C^{\Bmf}$ for all $C\in \ALCI(\Sigma)$.

\begin{lemma}\label{lem:equivalence}
  Let $\Amf,d$ and $\Bmf,e$ be pointed structures
  of finite outdegree and $\Sigma$ a signature. 
  Then
  $$
  \Amf,d \equiv_{\mathcal{ALCI},\Sigma} \Bmf,e \text{ iff }
  \Amf,d \sim_{\mathcal{ALCI},\Sigma}\Bmf,e.
  $$
  For the ``if''-direction, the condition on the outdegree can be
  dropped.
\end{lemma}

% \bigskip  

% {\color{blue}do we need fundamental lemmas / properties of
%   homomorphisms?  do we want to link them to CQs already here?}

%this paper homomorphisms and isomorphisms preserve relations but not functions (and so not constants). 
%	All languages are without function symbols (and so constants). We first do \emph{not} make
%        the unique name assumption. This turns out to have a very significant impact on the chacterizations
%        (for FO and GF) and also on decidability. It seems everything goes much smoother without the UNA, which
%        is probably worth discussing in a paper.

For any syntactic object $O$ such as a formula, an ontology, and a KB, we use
$\text{sig}(O)$ to denote the set of relation symbols that occur in $O$
and $||O||$ to denote the \emph{size} of $O$, that is, the number of
symbols needed to write it with names of relations, variables, and
constants counting as a single symbol.

\section{Weak Separability With Signature}

\newcommand{\mLOKB}{labeled $\Lmc$-KB}

We start with introducing the problem of (weak) separability with
signature, in its
projective and non-projective version. 
%
% learning instances and formulas separating positive and negative examples 
% as solutions to learning instances. We characterize separating formulas for KBs given in FO. 
% In particular, we show that UCQs separate the same instances as full FO
% if all symbols from the database are admitted in separating formulas.
% Deciding the existence of a separating FO formula is then shown to be equivalent to evaluating 
% rooted UCQs on KBs. The complexity of deciding separability is thus shown to be closely
% linked to evaluating UCQs. 
%
\newcommand{\LmcO}{\Lmc}
\begin{definition}
	\label{def:separa}
	Let $\LmcO$ be a fragment of FO. 
	A \emph{\mLOKB} takes the form
	$(\Kmc,P,N)$ with $\Kmc=(\Omc,\Dmc)$ an $\LmcO$-KB and
	$P,N\subseteq \text{cons}(\Dmc)^{n}$ non-empty sets of \emph{positive}
	and \emph{negative examples}, all of them tuples of the same length
	$n$. % called the \emph{arity} of $(\Kmc,P,N)$.
	
	%Let $\Lmc_S$ be a fragment of FO.  
	%
	Let $\Sigma \subseteq \text{sig}(\Kmc)$ be a signature. 
	An FO$(\Sigma)$-formula $\varphi(\vec{x})$ with $n$ free variables
	\emph{$\Sigma$-separates $(\Kmc,P,N)$} if $\text{sig}(\Kmc)\cap \text{sig}(\varphi)\subseteq \Sigma$ and 
	\begin{enumerate}
		
		\item 
		$\Kmc\models \varphi(\vec{a})$ for all $\vec{a}\in P$ and 
		
		\item $\Kmc\not\models \varphi(\vec{a})$ for all $\vec{a}\in N$.
		
	\end{enumerate}
	%
	% We say that $\varphi(\vec{x})$ \emph{non-projectively separates
	%   $(\Kmc,P,N)$} if $\mn{sig}(\varphi) \subseteq
	% \mn{sig}(\Kmc)$ and that it \emph{projectively separates} otherwise.
	%
	% Let $\Lmc_S$ be a fragment of FO. $(\Kmc,P,N)$ is
	% \emph{(projectively) $\Lmc_S$-separable} if there is an
	% $\Lmc_S$-formula that (projectively) separates $(\Kmc,P,N)$.
	Let $\Lmc_S$ be a fragment of FO. We say that $(\Kmc,P,N)$ is
        \emph{projectively $\Lmc_S(\Sigma)$-separable} if there is an
        $\Lmc_S$-formula $\varphi(\vec{x})$ that
        $\Sigma$-separates $(\Kmc,P,N)$ and \emph{(non-projectively)
          $\Lmc_S(\Sigma)$-separable} if there is such a
        $\varphi(\vec{x})$ with $\text{sig}(\varphi) \subseteq \Sigma$.
	%
	%$\varphi(\vec{x})$ is then called a \emph{separating formula} for $(\Kmc,P,N)$.
\end{definition}
Relation symbols in $\Sigma$-separating formulas that are not from
$\Sigma$ should be thought of as helper symbols.  % In
% applications, it might be more desirable to find separating formulas
% that do not use such symbols. However,
Their availability sometimes makes inseparable KBs separable,
examples are provided below.
We only consider FO-fragments $\Lmc_{S}$ that are closed under
conjunction.  In this case, a labeled KB $(\Kmc,P,N)$ is
$\Lmc_{S}(\Sigma)$-separable if and only if all
$(\Kmc,P,\{\vec{b}\})$, $\vec{b}\in N$, are
$\Lmc_{S}(\Sigma)$-separable, and likewise for projective
$\Lmc_{S}(\Sigma)$-separability, see \cite{KR}. % This is easy to see: to obtain a
% formula that separates $(\Kmc,P,N)$, take the conjunction of the
% formulas that separate $(\Kmc,P,\{\vec{b}\})$, $\vec{b}\in N$.  
In
 what follows, we thus mostly consider labeled KBs with singleton sets
 $N$ of negative examples.

% We distinguish three different conditions on the signature of the
% separating formula: the \emph{full signature case} in which
% $\Sigma=\Sigma_{\text{full}}$, the \emph{KB signature case} in which
% $\Sigma= {\sf sig}(\Kmc)$, and the \emph{arbitrary signature case} in
% which symbols from an arbitrary signature $\Sigma$ are admitted.
%
 Each choice of an ontology language $\Lmc$ and a separation language
 $\Lmc_{S}$ give rise to a projective and to a non-projective
 separability problem. In the current paper, we only consider cases
 where $\Lmc=\Lmc_S$ (see below for a discusssion).
\begin{center}
	\fbox{\begin{tabular}{@{\;}ll@{\;}}
			\small{PROBLEM} : & (Projective)
                        $\Lmc$-separability with signature\\
			{\small INPUT} : &  A \mLOKB\xspace $(\Kmc,P,N)$ and signature $\Sigma\subseteq \text{sig}(\Kmc)$ \\
			{\small QUESTION} : & Is $(\Kmc,P,N)$ (projectively) 
			$\Lmc(\Sigma)$-separable?  \end{tabular}}
\end{center}
\smallskip
\noindent
We study the \emph{combined complexity} of \Lmc-separability with 
signature where the ontology \Omc, database \Dmc (both in \Kmc), and 
sets of examples $P$ and $N$ are all taken to be part of the 
input. One can also study \emph{data complexity} where only $\Dmc$, 
$P$, and $N$ are regarded as inputs while $\Omc$ is assumed to be 
fixed \cite{KR}.  A special case of $\Lmc$-separability with 
signature is $\Lmc$-\emph{definability} with signature where $P$ and 
$N$ partition the example space, that is, inputs are {\mLOKB}s 
$(\Kmc,P,N)$ such that $N=\text{cons}(\Dmc)^n \setminus P$, $n$ the 
length of example tuples. % when $\Kmc=(\Tmc,\Amc)$. 
All our results also hold for definability. 

\medskip

We now give some examples illustrating the central notions used
in this paper. 
In~\cite{KR}, projective and non-projective separability are
studied without signature restrictions. Thus, all symbols used in 
the KB can appear in separating formulas. Rather surprizingly,
it turned out that in this case many different separation languages have exactly the same separating power. For example,
a labeled FO-KB turned out to be FO-separable iff it is UCQ-separable\footnote{We denote by UCQ the set of FO-formulas that
	are disjunctions of formulas constructed from atoms using conjunction and existential quantification.} 
(and projective and non-projective separability coincide) and
a labeled $\mathcal{ALCI}$-KB is projectively $\mathcal{ALCI}$-separable iff it is (non-)projectively FO-separable.
No such result can be expected for separability with signature restrictions, as illustrated by the following example. 
\begin{example}
	Let $\Omc = \{A \sqsubseteq \exists R.B \sqcap \exists R.\neg B\}$ and $\Dmc= \{A(a),R(b,c)\}$. Let $P=\{a\}$, $N=\{b\}$,
	and $\Sigma=\{R\}$. Clearly, the formula
	$$
	\exists y \exists y' (R(x,y) \wedge R(x,y') \wedge \neg (y=y'))
	$$
	$\Sigma$-separates $(\Omc,\Dmc,P,N)$, but $(\Omc,\Dmc,P,N)$ is neither UCQ$(\Sigma)$-separable nor $\mathcal{ALCI}(\Sigma)$-separable.
\end{example}
However, the ability to restrict separating formulas to a given
signature makes it possible to guide separation towards desired
aspects.
\begin{example}
  Consider a KB about books that uses, say, the concept and role names
  provided by schema.org (called types and properties
  there). Schema.org offers dozens of such names related to books,
  ranging from editor, author, and illustrator to genre, date
  published, and character. Assume that a few books have been labeled
  as \emph{likes} (added to $P$) and \emph{dislikes} (added to $N$)
  and one would like to find a formula $\vp$ that separate $P$ from
  $N$. Then it might be useful to restrict the signature of $\varphi$
  so as to concentrate on the aspects of books that one is most
  interested in. For example, one could select a signature that
  contains symbols related to genre such as graphic novel, adventure,
  classic, and drop all remaining symbols. If no separating formula
  exists, one can then iteratively extend the signature until
  separating formulas are found.  We refer the reader to research on
  modules and modularity in ontologies, where signatures are also used
  to capture the topic of a
  module~\cite{HKGrauS08,KonevLWW09,DBLP:journals/ai/BotoevaKRWZ16,DBLP:journals/ai/BotoevaLRWZ19}.
  If one is not sure which aspects are most relevant for separation,
  one might of course also decide to work with a large signature. But
  also in such a case, it might be useful to exclude certain
  undersired symbols such as the author's age and gender.
\end{example}
The helper symbols that distinguish the projective from the
non-projective case play a completely different role from the
symbols in the signature $\Sigma$ selected for separation,
as discussed next.

\begin{example}\label{exp:gg}
Consider a database $\Dmc$ in which an individual $a$ is part of an $R$-cycle and $b$ has both an $R$-reflexive successor and predecessor. Thus,
$$
\Dmc= \{R(a_{0},a_{1}), \ldots, R(a_{n-1},a_{n}),
R(b,b_{1}),R(b_{1},b_{1}),R(b_{2},b),R(b_{2},b_{2})\},
$$
where $a_{0}=a_{n}=a$ and $n>0$. 
Let $\Omc$ be either empty or any $\ALCI$-ontology such that
$\Kmc=(\Omc,\Dmc)$ is satisfiable and $\top \sqsubseteq \exists R.\top
\sqcap \exists R^{-}.\top\in \Omc$. Further let $P=\{a\}$,
$N=\{b\}$, and $\Sigma=\{R\}$. If no helper symbols are allowed,
then $(\Kmc,P,N)$ is not $\mathcal{ALCI}(\Sigma)$-separable
(we show this in Example~\ref{exmp:tt} below). If, however, a helper symbol $A$ is allowed, then $\neg A \sqcup \exists R^{n}.A$ is a separating $\mathcal{ALCI}$-concept. Thus,
$(\Kmc,P,N)$ is projectively $\mathcal{ALCI}(\Sigma)$-separable but not non-projectively $\mathcal{ALCI}(\Sigma)$-separable.
The use of a cycle in this example is no accident; we show 
in the next section that if there are no cycles in the database, then helper symbols do not add any separating power to $\mathcal{ALCI}$-concepts.
\end{example}

%{\color{orange} succinctness?}

\section{Weak  Separability in \ALCI}\label{sec:alciweak}

We give a model-theoretic characterization of projective weak
$\mathcal{ALCI}$-separability with signature and use it to prove
decidability in \TwoExpTime. A matching lower bound is obtained by
reduction from conservative extensions. We also use the
characterization to discuss the relationship between non-projective
and projective separability, showing, in particular, that on
tree-shaped databases the two notions coincide.
% helper symbols of the projective approach do
% not make a difference and the two approaches separate the same labeled
% KBs.  We then prove that projective $\mathcal{ALCI}$-separability with
% signatures is \TwoExpTime-complete, which is the main result of this
% section.

We start with a model-theoretic characterization of projective
$\ALCI(\Sigma)$-separability of labeled KBs.
As \ALCI-concepts talk about individual elements and not 
tuples, we assume in this section that examples in labeled KBs are 
constants from the database. 
In fact, we give three characterizations that are more and more refined. We use the second one to clarify the relationship between 
projective and non-projective separability and the 
third characterization for the decision procedure. The first
characterization directly reflects that we are considering projective
separability with signature as it is based on
$\ALCI(\Sigma')$-bisimulations where $\Sigma'$ is a signature
satisfying $\Sigma' \cap \text{sig}(\Kmc)\subseteq \Sigma$. The second
characterization replaces $\ALCI(\Sigma')$-bisimulations by functional
$\ALCI(\Sigma)$-bisimulations. The final characterization replaces
the functional bisimulations from the second characterization by a combination of
$\Sigma$-homomorphisms and $\ALCI(\Sigma)$-bisimulations.  We first
introduce some required notation.

An \emph{extended database} is a database that additionally may
contain `atoms' of the form $C(a)$ with $C$ an \ALCI-concept. The
semantics of extended databases is defined in the expected way.  We
write $\Amf,a \sim_{\ALCI,\Sigma}^{f} \Bmf,b$ if there exists an
$\ALCI(\Sigma)$-bisimulation $S$ between $\Amf$ and $\Bmf$ that
contains $(a,b)$ and is functional, that is, $(d,d_{1}),(d,d_{2})\in
S$ imply $d_{1}=d_{2}$. Let $\text{sub}(\Kmc)$ denote the set of
concepts that occur in $\Kmc$, closed under single negation and under
subconcepts. The \emph{$\Kmc$-type realized in a pointed structure
  $\Amf,a$} is defined as
$$
\text{tp}_{\Kmc}(\Amf,a) = \{ C\in \text{sub}(\Kmc) \mid a\in C^{\Amf}\}.
$$
A \emph{$\Kmc$-type} is any set $t \subseteq \text{sub}(\Kmc)$ of the
form $\text{tp}_{\Kmc}(\Amf,a)$.  For a pointed database $\Dmc,a$, we
write $\Dmc_{\text{con}(a)},a \rightarrow^{\Sigma}_{c} \Amf,b$ if
there is a $\Sigma$-homomorphism $h$ from the maximal connected
component $\Dmc_{\text{con}(a)}$ of $a$ in $\Dmc$ to $\Amf$ such that
$h(a)=b$ and there is a $\Kmc$-type $t_{d}$ for each $d\in
\text{cons}(\Dmc_{\text{con}(a)})$ such that:
\begin{enumerate}
	\item there exists a model $\Bmf_{d}$ of $\Omc$ with $\text{tp}_{\Kmc}(\Bmf_{d},d) = t_{d}$ and $\Bmf_{d},d\sim_{\ALCI,\Sigma} \Amf,h(d)$;
	\item $(\Omc,\Dmc')$ is satisfiable, for the extended database
          $\Dmc'=\Dmc\cup \{ C(d) \mid C \in t_d, \ d\in \text{cons}(\Dmc_{\text{con}(a)})\}$.
\end{enumerate}
\begin{theorem}\label{thm:alci-refinement}
  Let $(\Kmc,P,\{b\})$ be a labeled $\ALCI$-KB with $\Kmc=(\Omc,\Dmc)$
  and $\Sigma\subseteq \text{sig}(\Kmc)$.  Then the following
  conditions are equivalent:
  \begin{enumerate}

    \item $(\Kmc,P,\{b\})$ is projectively
      $\mathcal{ALCI}(\Sigma)$-separable.

    \item there exists a forest model $\Amf$ of $\Kmc$ of finite
      outdegree and a signature $\Sigma'$ such that  $\Sigma' \cap
      \text{sig}(\Kmc)\subseteq \Sigma$ and for all models $\Bmf$ of
      $\Kmc$ and all $a\in P$: $\Bmf,a^{\Bmf}
      \not\sim_{\ALCI,\Sigma'} \Amf, b^{\Amf}$.

    \item there exists a forest model $\Amf$ of $\Kmc$ of finite
      outdegree such that for all models $\Bmf$ of $\Kmc$ and all
      $a\in P$: $\Bmf,a^{\Bmf} \not\sim_{\ALCI,\Sigma}^{f} \Amf,
      b^{\Amf}$. 

    \item there exists a forest model $\Amf$ of $\Kmc$ of finite
      outdegree such that for all $a\in P$: $\Dmc_{\text{con}(a)},a
      \not\rightarrow^{\Sigma}_{c} \Amf,b^{\Amf}$.

  \end{enumerate}  
\end{theorem}
The proof relies on Lemmas~\ref{lem:forestmodelcompleteness} and
\ref{lem:equivalence} and is given in the appendix.

We first use Theorem~\ref{thm:alci-refinement} to discuss the
relationship between projective and non-projective separability.  A
basic model-theoretic characterization of non-projective
$\ALCI(\Sigma)$-separability is rather straightforward to obtain by
just dropping the quantification over $\Sigma'$ in Condition~2 of
Theorem~\ref{thm:alci-refinement} and demanding instead that
$\Bmf,a^{\Bmf} \not\sim_{\ALCI,\Sigma} \Amf, b^{\Amf}$, for all models
$\Bmf$ of $\Kmc$ and $a\in P$. As a consequence, one can then
also adapt Condition~3 of Theorem~\ref{thm:alci-refinement} to
the non-projective case by % : while for projective
% separability we demand the existence of a model $\Amf$ of $\Kmc$ such
% that for all models $\Bmf$ of $\Kmc$:
% $\Bmf,a^{\Bmf} \not\sim_{\ALCI,\Sigma}^{f} \Amf,b^{\Amf}$, for
% non-projective separability we
simply droppüing the functionality condition on the bisimulation.  If
the database $\Dmc_{\text{con}(a)}$ is tree-shaped, then there is
no difference between the two versions of Condition~2 as one
can always introduce sufficiently many copies of nodes in models
$\Bmf$ of $\Kmc$ to turn an unrestricted bisimulation into a functional
one. We obtain the following result.
\begin{theorem}
	Let $(\Kmc,P,N)$ be a labeled $\ALCI$-KB with $\Kmc=(\Omc,\Dmc)$ such that $\Dmc_{\text{con}(a)}$ is tree-shaped for all $a\in P$ and $\Sigma \subseteq \text{sig}(\Kmc)$. Then 
	$(\Kmc,P,N)$ is projectively $\mathcal{ALCI}(\Sigma)$-separable 
	iff it is non-projectively $\mathcal{ALCI}(\Sigma)$-separable.
\end{theorem}
The following example illustrates Theorem~\ref{thm:alci-refinement}.
\begin{example}\label{exmp:tt}
  Consider the labeled KB $(\Kmc,P,N)$ and signature $\Sigma$ from
  Example~\ref{exp:gg}.  Then we find a model $\Amf$ of $\Kmc$ of
  finite outdegree such that $b^{\Amf}$ does not participate in any
  $R$-cycle. Then, as $a$ participates in an $R$-cycle, there does not
  exist any model $\Bmf$ of $\Kmc$ such that
  $\Bmf,a^{\Bmf} \sim_{\ALCI,\Sigma}^{f} \Amf, b^{\Amf}$, and so
  $(\Kmc,P,N)$ is projectively $\mathcal{ALCI}(\Sigma)$-separable.
  $(\Kmc,P,N)$ is not non-projectively
  $\mathcal{ALCI}(\Sigma)$-separable: if $\Omc$ contains
  $\top \sqsubseteq \exists R.\top \sqcap \exists R^{-}.\top$, then
  $S=\text{dom}(\Bmf) \times \text{dom}(\Amf)$ is an
  $\ALCI(\Sigma)$-bisimulation between $\Bmf$ and $\Amf$, for any
  model $\Bmf$ of $\Kmc$. If $\Omc$ is empty, then the construction of
  the required bisimulation is also straightforward.
\end{example}	
We now come to the main result of this section.
\begin{theorem} \label{thm:alci-complexity}
	Projective $\ALCI$-separability with signature is \TwoExpTime-complete.
\end{theorem}

The upper bound in Theorem~\ref{thm:alci-complexity} is obtained by
using the characterization in Condition~4 of
Theorem~\ref{thm:alci-refinement} to devise a decision procedure based
on tree automata. Given a labeled $\ALCI$-KB $(\Kmc,P,\{b\})$, we
construct a tree automaton \Amc such that the language recognized by
\Amc is non-empty if and only if there is a forest model \Amf of \Kmc
as described in Condition~4 of Theorem~\ref{thm:alci-refinement}. We
use a variant of two-way alternating parity tree automata over
infinite trees~\cite{DBLP:conf/icalp/Vardi98}. In contrast to the
standard model, our automata work on trees of finite, but not
necessarily bounded outdegree.  Such an automata model has been
introduced in~\cite{JLMSW17}. We recall the technical preliminaries.

A \emph{tree} is a non-empty (and potentially infinite) set
of words $T \subseteq (\mathbb{N} \setminus 0)^*$ closed under prefixes.
We generally assume that trees are finitely branching, that is, for
every $w \in T$, the set $\{ i \mid w \cdot i \in T \}$ is finite.
For any $w \in (\mathbb{N} \setminus 0)^*$, as a convention we set $w
\cdot 0 := w$. If $w=n_0n_1 \cdots n_k$, 
we additionally
set $w \cdot -1 := n_0 \cdots n_{k-1}$.  For an alphabet $\Theta$, a
\emph{$\Theta$-labeled tree} is a pair $(T,L)$ with $T$ a tree and
$L:T \rightarrow \Theta$ a node labeling function.

A {\em two-way alternating tree automaton (2ATA)} is a tuple $\Amc =
(Q,\Theta,q_0,\delta,\Omega)$ where $Q$ is a finite set of {\em
states}, $\Theta$ is the finite {\em input alphabet}, $q_0\in Q$ is
the {\em initial state}, $\delta$ is a {\em transition function} as
specified below, and $\Omega:Q\to \mathbb{N}$ is a {\em priority
function}. 
The automaton runs on $\Theta$-labeled trees.  The transition function
maps a state $q$ and some input letter $\theta\in \Theta$ to a {\em
transition condition $\delta(q,\theta)$} which is a positive Boolean
formula over the truth constants $\mn{true}$ and $\mn{false}$ and
transitions of the form $q$, $\langle-\rangle q$, $[-] q$, $\Diamond
q$, $\Box q$ where $q \in Q$.  Informally, the transition $q$
expresses that a copy of the automaton is sent to the current node in
state $q$, $\langle - \rangle q$ means that a copy is sent in state
$q$ to the predecessor node, which is then required to exist, $[-] q$
means the same except that the predecessor node is not required to
exist, $\Diamond q$ means that a copy is sent in state $q$ to some
successor, and $\Box q$ that a copy is sent in state $q$ to all
successors. The semantics is defined in terms of runs in the usual
way, we refer to~\cite{JLMSW17} for details.  We use $L(\Amc)$ to
denote the set of all $\Theta$-labeled trees accepted by \Amc.  2ATAs
are closed under complementation and intersection, and their emptiness
problem, which asks whether $L(\Amc)=\emptyset$ for a given 2ATA \Amc,
can be decided in time exponential in the number of states
of~\Amc~\cite{JLMSW17}. 

The input alphabet $\Theta$ consists of two types
of symbols:
\begin{enumerate}

  \item models $\Amf_0$ of $\Dmc$ with
    $\text{dom}(\Amf_0) \subseteq D_0$, $D_0$ a fixed set of
    cardinality $|\text{cons}(\Dmc)|$;

  \item triples $(a,R,M)$ for $a\in\text{cons}(\Dmc)$, $R$ a role used in
    \Omc, and $M\subseteq \text{sig}(\Omc) \cap \NC$.

\end{enumerate}
A $\Theta$-labeled tree is \emph{well-formed} if it has a label of
Type~1 at the root and labels of Type~2 everywhere else. A well-formed
$\Theta$-labeled tree $\tau$ encodes a structure $\Amf_\tau$ that can be
constructed as follows:
\begin{itemize}

  \item start with $\Amf_\tau=\Amf_0$, the structure  from the root label; 

  \item for every non-root $v\in T$ with 
    $\tau(v)=(a,R,M)$, extend $\Amf_\tau$ as follows:
    \begin{itemize}

    \item if the predecessor of $v$ is the root, add an $R$-successor
      $a_v$ of~$a^{\Amf_\tau}$;

    \item if the predecessor $v'$ of $v$ is not the root, add an
      $R$-successor $a_v$ of $a_{v'}$ (so $a$ is ignored in this case
      and should be considered a dummy.)

    \end{itemize}
    In both cases, $a_v$ makes true exactly the concept names in $M$.
\end{itemize}
It should be clear that every structure $\Amf_\tau$ is a forest
structure.  Conversely, we can encode every forest structure \Amf into
a $\Theta$-labeled tree $\tau$ such that $\Amf_\tau=\Amf$.
\begin{lemma} \label{lem:automata}
  There are 2ATAs $\Amc_1,\Amc_2,\Amc_3$ such that: 
  \begin{enumerate}

    \item $\Amc_1$ accepts precisely the well-formed $\Theta$-labeled
      trees; 

    \item $\Amc_2$ accepts a well-formed $\Theta$-labeled tree
      $\tau$ iff $\Amf_\tau$ is a model of \Kmc;

    \item $\Amc_3$ accepts a well-formed $\Theta$-labeled tree $\tau$
      if $\Dmc_{\text{con}(a)},a \rightarrow^{\Sigma}_{c}
      \Amf_\tau,b^{\Amf_\tau}$.

  \end{enumerate}
  The number of states of $\Amc_1$ is $2$; the number of states of
  $\Amc_2$ is polynomial in $||\Omc||$; the number of states of 
  $\Amc_3$ is exponential in $||\Kmc||$. Moreover,
  $\Amc_1,\Amc_2,\Amc_3$ can be constructed in time double exponential
  in $||\Kmc||$.
\end{lemma}
As the construction of $\Amc_1$ and $\Amc_2$ is rather standard, we only
sketch the construction of the automaton $\Amc_3$. See
e.g.~\cite{JLMSW17}
for full details of a similar construction.

As the first step,
$\Amc_3$ reads the symbol $\Amf_0$ at the root and
non-deterministically guesses the following: % where $\text{dom}(\Amf)=\{a_1,\dots,a_m\}$:
\begin{itemize}

  \item types $t_d$, $d\in\text{cons}(\Dmc_{\text{con}(a)})$, such that $(\Omc,\Dmc')$
    is satisfiable where  $\Dmc'=\Dmc\cup \{ C(d) \mid C \in t_d, \ d\in \text{cons}(\Dmc_{\text{con}(a)})\}$;

  \item a partition $\Dmc_0,\Dmc_1,\ldots,\Dmc_m$ of
    $\Dmc_{\text{con}(a)}$
    % ,
    % $m \leq |\text{cons}(\Dmc)|$,
    such that $\text{cons}(\Dmc_i)\cap\text{cons}(\Dmc_j) =
    \emptyset$ for 
$1\leq i<j\leq m$;
    % %
    % \begin{enumerate}[label=(\alph*)]

    %   % \item there is a homomorphism $h_0$ witnessing
    %   %   $\Dmc_{\text{con}(a)}\to \bigcup_{i}\Dmc_i$ and the $\Dmc_i$
    %   %   are minimal with this property;

    %   % \item $\Dmc_i$ is a tree without multi-edges whenever $i>0$;

    %   \item $\text{cons}(\Dmc_i)\cap\text{cons}(\Dmc_j)=\emptyset$ for
    %     all $i,j$ with $1\leq i<j\leq m$;

    %   \item $\text{cons}(\Dmc_0)\cap\text{cons}(\Dmc_i)$ is a
    %     singleton $\{a_i\}$ for all $i$ with $1\leq i\leq m$;

    % \end{enumerate}

  \item a $\Sigma$-homomorphism $h$ from $\Dmc_0$ to
    $\Amf_0$ such that $h(a)=b^{\Amf_0}$ and there are
    $a_1,\dots,a_m$ with $h(c)=a_i$
    for all $c \in \text{cons}(\Dmc_0 \cap \Dmc_i)$ and
    $1 \leq i \leq m$.

\end{itemize}
The entire guess is stored in the state of the automaton. Note that
the first item checks that Point~2 from the definition of
$\Dmc_{\text{con}(a)},a \rightarrow^{\Sigma}_{c}
\Amf_\tau,b^{\Amf_\tau}$ is satisfied for the guessed types
$t_d$. After making its guess, the automaton verifies that the
homomorphism $h$ from the last item can be extended to a homomorphism
from $\Dmc_{\text{con}(a)}$ to $\Amf_\tau$ that satisfies Point~1 from
the definition of
$\Dmc_{\text{con}(a)},a \rightarrow^{\Sigma}_{c}
\Amf_\tau,b^{\Amf_\tau}$. To this end, it does a top-down traversal of
$\Amf_\tau$ checking that each $\Dmc_i$ can be homomorphically mapped
to the subtree of $\Amf_\tau$ below $h(a_i)$. During the traversal,
the automaton memorizes in its state the set of constants from $\Dmc_i$ that
are mapped to the currently visited element.

% Intuitively, the databases $\Dmc_0,\ldots,\Dmc_m$ describe the image
% of $h$ of $\Dmc_{\text{con}(a)}$ into the forest model $\Amf_\tau$. In
% particular, $\Dmc_0$ describes the part of $\Dmc_{\text{con}(a)}$
% mapped into $\Amf_0$ while the $\Dmc_i$, $i>0$ describe the part of
% $\Dmc$ mapped into the subtrees below $a_i^{\Amf_\tau}$ in
% $\Amf_\tau$.  Since $g_0$ is a $\Sigma$-homomorphism from $\Dmc_0$ to
% $\Amf_0$, by the third item above,
% It suffices to verify that there
% are $\Sigma$-homomorphisms $g_i$ from each $\Dmc_i$ with $i>0$ into
% $\Amf_\tau$. These homomorphisms can be non-deterministically guessed
% because each $\Dmc_i$ is a tree.  The composition of $h_0$ with the
% $\Sigma$-homomorphisms $g_i$ is the required $\Sigma$-homomorphism $h$
% from $\Dmc_{\text{con}(a)}$ to $\Amf_\tau$.

The automaton additionally makes sure that Point~1 from the definition
of
$\Dmc_{\text{con}(a)},a \rightarrow^{\Sigma}_{c}
\Amf_\tau,b^{\Amf_\tau}$ is satisfied, in the following way. During
the top-down traversal, it spawns copies of itself to verify that,
whenever it has decided to map a
$d \in \text{cons}(\Dmc_{\text{con}(a)})$ to the current element, then
there is a tree-shaped model $\Bmc_d$ of \Omc with
$\text{tp}_{\Kmc}(\Bmf_{d},d) = t_{d}$ and a bisimulation that
witnesses $\Bmc_d,d \sim_{\ALCI,\Sigma} \Amf_\tau,c$. This is done by
`virtually' traversing $\Bmc_d$ elements-by-element, storing at each
moment only the type of the current element in a state. This is
possible because $\Bmf_d$ is tree-shaped. At the beginning, the
automaton is at an element of $\Bmc_d$ of type $t_d$ and knows that
the bisimulation maps this element to the node of $\Amf_\tau$
currently visited by the automaton. It then does two things to verify
the two main conditions of bisimulations.  First, it transitions to
every neighbor of the node of $\Amf_\tau$ currently visited, both
upwards and downwards, and carries out in its state the corresponding
transition in $\Bmc_d$, in effect guessing a new type. Second, it
considers the current type of $\Bmc_d$ and guesses successor types
that satisfy the existential restrictions in it. For every required
successor type,  it then guesses a neighbor of the currently visited
node in $\Amf_\tau$ to which the successor is mapped. The two steps
are alternated, exploiting the alternation capabilities of the
automaton. Some extra bookkeeping in states is needed for the root
node of the input tree as it represents more then one element
of $\Amf_\tau$.

% Now for Point~1. Observe first that the model $\Bmf_d$ from Point~1
% can assumed to be a tree model such that the required
% $\ALCI(\Sigma)$-bisimulation always `walks down' in the tree $\Bmf_d$.
% So, whenever the automaton has guessed $h(d)$ for some~$d$, the
% automaton can verify the existence of the bisimulations as follows.
% It keeps in its states the type of the node of $\Bmf_d$ and the
% element of $\Amf_\tau$ which are supposed to bisimilar, starting with
% $t_d,h(d)$. The bisimulation conditions can now be directly
% implemented by the availability of alternation in the automata model.
% For example, let the automaton be in state $t,e$. The automaton
% verifies first that $t$ and $e'$ coincide on all $\Sigma$-concept
% names. To verify the back and forth conditions of the bisimulation,
% note that the encoding of $\Amf_\tau$ allows the automaton to identify
% locally all $R$-successors of $e$, $R$ a $\Sigma$-role.  For the forth
% condition, the automaton sends a copy to all such $R$-successors $e'$
% of $e$, guesses a type $t'$, and enters state $t',e'$. The back
% condition is symmetric.

It can be verified that only exponentially many states are required
and that the transition function can be computed in double exponential
time. This finishes the proof sketch of Lemma~\ref{lem:automata}.

\smallskip The upper bound in Theorem~\ref{thm:alci-complexity} is now
obtained as follows. By Condition~4 of
Theorem~\ref{thm:alci-refinement} and Lemma~\ref{lem:automata},
$(\Kmc,P,\{b\})$ is projectively $\ALCI(\Sigma)$-separable iff
$L(\Amc_1)\cap L(\Amc_2)\cap \overline{L(\Amc_3)}$ is not empty where
$\overline{L(\Amc_3)}$ denotes the complement of $L(\Amc_3)$. By
Lemma~\ref{lem:automata}, all $\Amc_i$ can be constructed in double
exponential time and their number of states is single exponential in
$||\Kmc||$. As the complement and intersections of 2ATAs can be
computed in polynomial time with only a polynomial increase in the
number of states, it remains to recall that non-emptiness of
2ATAs can be decided in time exponential in the number of states.

\medskip
%
% We have thus established that projective $\ALCI$-separability with
% signature is in \TwoExpTime. 
For the lower bound, we reduce from
conservative extensions in \ALCI. An \ALCI-ontology $\Omc_2$ is a
\emph{conservative extension} of an \ALCI-ontology $\Omc_1 \subseteq
\Omc_2$ if there is no $\ALCI(\text{sig}(\Omc_1))$-concept $C$ that is
satisfiable w.r.t.\ $\Omc_1$, but unsatisfiable w.r.t.\ $\Omc_2$. We
define \emph{projective conservative extensions} in the same way
except that $C$ is now an \ALCI-concept with $ \text{sig}(\Omc_2) \cap
\text{sig}(C) \subseteq \text{sig}(\Omc_1)$.  It was
shown in \cite{DBLP:conf/kr/GhilardiLW06} that it is \TwoExpTime-hard to decide,
given \ALCI-ontologies $\Omc_1$ and $\Omc_2$, whether $\Omc_2$ is a
(non-projective) conservative extensions of $\Omc_1$. It was further
observed that conservative extensions and projective
conservative extensions coincide in logics that enjoy Craig
interpolation, which \ALCI does~\cite{JLMSW17}. Thus, projective conservative
extensions in \ALCI are also \TwoExpTime-hard.  We give a polynomial
time reduction from (the complement of) that problem to projective
\ALCI-separability with signature.

Thus let $\Omc_1, \Omc_2$ be \ALCI-ontologies with
$\Omc_1 \subseteq \Omc_2$. We can assume w.l.o.g.\ that $\Omc_1$ takes
the form $\{ \top \sqsubseteq C_1 \}$,  $\Omc_2$ takes the form
$\Omc_1 \cup \{ \top \sqsubseteq C_2 \}$, and that $\Omc_2$ is satisfiable. %  where $C_1$ and $C_2$ are
% in negation normal form (NNF), that is, negation is only applied to
% concept names, but not to compound concepts.
For a concept name $A$, the \emph{$A$-relativization} $C^A$ of an
\ALCI-concept $C$ %in NNF
is obtained by replacing every subconcept $\exists r . D$ in $C$ with
$\exists r . (A \sqcap D)$. Define
$$\Omc = \{ \top \sqsubseteq C^{A_1}_1, A \sqsubseteq C^{A}_2 \}
$$ where
$A$ is a concept name that does not occur in $\Omc_2$.  Then
$\Omc_2$ is not a projective conservative extension of $\Omc_1$ iff
there is an $\ALCI(\text{sig}(\Omc_1))$-concept $C$ that is
satisfiable w.r.t.\ $\Omc_1$ but unsatisfiable w.r.t.\ $\Omc_2$ iff
the labeled \ALCI-KB $(\Kmc,\{a\},\{b\})$ is projectively
$\ALCI(\text{sig}(\Omc_1))$-separable by $\neg C$ where
$\Kmc=(\Omc,\Dmc)$ and $\Dmc = \{ A(a), D(b) \}$, $D$ a
fresh (dummy) concept name. We have thus
established the lower bound from Theorem~\ref{thm:alci-complexity}.
% Thus let $\Omc_1, \Omc_2$ be \ALCI-ontologies with
% $\Omc_1 \subseteq \Omc_2$. We can assume w.l.o.g.\ that $\Omc_1$ takes
% the form $\{ \top \sqsubseteq C_1 \}$,  $\Omc_2$ takes the form
% $\Omc_1 \cup \{ \top \sqsubseteq C_2 \}$, and that $\Omc_2$ is satisfiable. %  where $C_1$ and $C_2$ are
% % in negation normal form (NNF), that is, negation is only applied to
% % concept names, but not to compound concepts.
% For a concept name $A$, the \emph{$A$-relativization} $C^A$ of an
% \ALCI-concept $C$ %in NNF
% is obtained by replacing every subconcept $\exists r . D$ in $C$ with
% $\exists r . (A \sqcap D)$. Define
% $$\Omc = \{ A_1 \sqsubseteq C^{A_1}_1, A_2 \sqsubseteq C^{A_2}_1
% \sqcap C^{A_2}_2 \}
% $$ where
% $A_1,A_2$ are concept names that do not occur in $\Omc_2$.  Then
% $\Omc_2$ is not a projective conservative extension of $\Omc_1$ iff
% there is an $\ALCI(\text{sig}(\Omc_1))$-concept $C$ that is
% satisfiable w.r.t.\ $\Omc_1$ but unsatisfiable w.r.t.\ $\Omc_2$ iff
% the labeled \ALCI-KB $(\Kmc,\{a\},\{b\})$ is projectively
% $\ALCI(\text{sig}(\Omc_1))$-separable by $\neg C$ where
% $\Kmc=(\Omc,\Dmc)$ and $\Dmc = \{ A_2(a), A_1(b) \}$. We have thus
% established the lower bound in the following result.
%
%\begin{theorem} \label{thm:alci-complexity-old}
%  %
%  Projective $\ALCI$-separability with signature is 
%  \TwoExpTime-complete.\nb{J: moved to the beginning of the section}
%  %
%\end{theorem}
%
\medskip

We leave open the decidability and exact complexity of non-projective
\ALCI-separability with signature. A \TwoExpTime lower bound can be
established along the lines above.

\section{Weak Separability: Undecidable Cases}

Inspired by the close connection of conservative extensions and weak
separability with signature that was established in the previous
section, we investigate logics for which
conservative extensions are undecidable: the guarded fragment GF and
the expressive DL $\mathcal{ALCFIO}$. The latter logic,
$\mathcal{ALCFIO}$, is the extension of \ALCI with \emph{nominals} and
\emph{functionality assertions}. Nominals are concepts of the form
$\{o\}$ with $o$ a constant symbol, and the translation $\cdot^\dagger$ from \ALCI into FO can be
extended to nominals by setting 
$\{o\}^\dagger=(x=o)$. Functionality assertions are concept inclusions
of the form $\top\sqsubseteq (\leqslant 1\ r)$, $r$ a role, which
demand that the role $r$ is interpreted as a partial function. Thus, they are
a weak form of counting, and indeed $\mathcal{ALCFIO}$ is a fragment
of $\text{C}^2$, the two-variable fragment of FO with counting. 

It is known that conservative extensions and projective conservative
extensions are undecidable in every extension of the three-variable
fragment GF$^3$ of GF~\cite{JLMSW17} and in every extension of
$\mathcal{ALCFIO}$~\cite{DBLP:conf/ijcai/LutzWW07}. Unfortunately, it
is not clear how to achieve a direct reduction of conservative
extensions to separability for both GF and $\mathcal{ALCFIO}$. In both
cases, the relativization that was used for \ALCI does not work. For
GF, this is the case because non-conservativity in GF is witnessed by
\emph{sentences} while separability is witnessed by \emph{formulas}.
For $\mathcal{ALCFIO}$, relativization cannot be applied due to the
presence of nominals/constants. We instead directly use and adapt the
strategies of the mentioned undecidability proofs, starting with GF.

\begin{theorem}\label{thm:main-gf}
  Projective and non-projective \Lmc-separability with signature are
  undecidable for every logic $\Lmc$ that contains $\text{GF}^{\,3}$.
  This is even true when the language of the separating formula is
  \ALC.
\end{theorem}
%
% It is known that conservative extensions and projective conservative
% extensions are undecidable in GF$^3$, in GF, and in
% GNF~\cite{JLMSW17}. Unfortunately, in contrast to $\ALCI$, a direct
% reduction of conservative extensions to separability seems
% impossible, since non-conservativity in GF is witnessed by
% \emph{sentences} while separability is witnessed by \emph{formulas}.
% We instead adapt the proof strategy of the mentioned undecidability
% result and 
The proof is by a reduction from the halting problem of two-register
machines. A (deterministic) \emph{two-register machine (2RM)}\index{Two-register
machine}\index{2RM} is a pair $M=(Q,P)$ with $Q = q_0,\dots,q_{\ell}$
a set of \emph{states} and $P = I_0,\dots,I_{\ell-1}$ a sequence of
\emph{instructions}.  By definition, $q_0$ is the \emph{initial
state}, and $q_\ell$ the \emph{halting state}. For all $i < \ell$,
\begin{itemize}

  \item either $I_i=+(p,q_j)$ is an \emph{incrementation instruction}
    with $p \in \{0,1\}$ a register and $q_j$ the subsequent state;

  \item or $I_i=-(p,q_j,q_k)$ is a \emph{decrementation instruction}
    with $p \in \{0,1\}$ a register, $q_j$ the subsequent state if
    register~$p$ contains~0, and $q_k$ the subsequent state otherwise.

\end{itemize}
A \emph{configuration} of $M$ is a triple $(q,m,n)$, with $q$ the
current state and $m,n \in \mathbb{N}$ the register contents.  We
write $(q_i,n_1,n_2) \Rightarrow_M (q_j,m_1,m_2)$ if one of the
following holds:
\begin{itemize}
	
	\item $I_i = +(p,q_j)$, $m_p = n_p +1$, and $m_{1-p} =
	n_{1-p}$;
	
	\item $I_i = -(p,q_j,q_k)$, $n_p=m_p=0$, and $m_{1-p} =
	n_{1-p}$;
	
	\item $I_i = -(p,q_k,q_j)$, $n_p > 0$, $m_p = n_p -1$, and
	$m_{1-p} = n_{1-p}$.
	
\end{itemize}
The \emph{computation} of $M$ on input $(n,m) \in \mathbb{N}^2$ is
the unique longest configuration sequence $(p_0,n_0,m_0)
\Rightarrow_M (p_1,n_1,m_1) \Rightarrow_M \cdots$ such that $p_0 =
q_0$, $n_0 = n$, and $m_0 = m$.
The halting problem for 2RMs is to decide, given a 2RM $M$, whether
its computation on input $(0,0)$ is finite (which implies that its
last state is $q_\ell$).

We convert a given 2RM $M$ into a labeled GF$^3$ KB
$(\Kmc,\{a\},\{b\})$, $\Kmc = (\Omc,\Dmc)$ and signature $\Sigma$ such that $M$ halts iff
$(\Kmc,\{a\},\{b\})$ is (non-)projectively GF$(\Sigma)$-separable iff
$(\Kmc,\{a\},\{b\})$ is (non-)projectively $\ALC(\Sigma)$-separable.
Let $M=(Q,P)$ with $Q = q_0,\dots,q_{\ell}$ and
$P = I_0,\dots,I_{\ell-1}$. We assume w.l.o.g.\ that $\ell \geq 1$ and
that if $I_i=-(p,q_j,q_k)$, then $q_j \neq q_k$. In $\Kmc$,
we use the following set of relation symbols:
\begin{itemize}

  \item a binary symbol $N$ connecting a configuration to
    its successor configuration;

  \item binary symbols $R_{1}$ and $R_{2}$ that represent the
    register contents via the length of paths;

  \item unary symbols $q_{0},\ldots,q_{\ell}$ representing the states of $M$;

  \item a unary symbol $S$ denoting points where a computation starts.

  \item a unary symbol $D$ used to represent that there is 
    some defect;

  \item binary symbols $D^{+}_{p},D_{p}^{-},D_{p}^{=}$ used to
    describe defects in incrementing, decrementing, and keeping
    register $p\in \{0,1\}$;

  \item ternary symbols
    $H_{1}^{+},H_{2}^{+},H_{1}^{-},H_{2}^{-},H_{1}^{=},H_{2}^{=}$ used
    as guards for existential quantifiers.

\end{itemize}
The signature $\Sigma$ consists of the symbols from the first four
points above.

We define the ontology $\Omc$ as the set of several GF$^3$
sentences.\footnote{The formulas that are not syntactically guarded
can easily be rewritten into such formulas.} The first sentence
initializes the starting configuration:
\[\forall x  (Sx \rightarrow (q_0x \wedge \neg \exists y \, R_0xy
\wedge \neg \exists y \, R_1xy))\]
Second, whenever $M$ is not in the final state, there is a
next configuration with the correctly updated state. For $0 \leq i <
\ell$, we include: 
$$ \begin{array}{rl} 
  \forall x (q_ix \rightarrow \exists y\, Nxy) & \\
  \forall x (q_{i}x \rightarrow \forall y (Nxy \rightarrow q_{j}y)) &
  \text{ if } I_i=+(p,q_j) \\[1mm] 
  \forall x ((q_{i}x \wedge \neg \exists y R_{p}xy) \rightarrow\forall
  y (Nxy \rightarrow q_{j}y)) & \text{ if } I_i=-(p,q_j,q_k) \\[1mm] 
  \forall x ((q_{i}x \wedge \exists y R_{p}xy)  \rightarrow \forall y
  (Nxy \rightarrow q_{k}y)) & \text{ if } I_i=-(p,q_j,q_k) \end{array} 
$$
Moreover, if $M$ is in the final state, there is no successor
configuration:
\[\forall x(q_\ell x\rightarrow \neg \exists y\,Nxy).\]
%
% and, similarly, the endmarker $E$ makes sure that the counters are
% not extended beyond it: % \[\forall x( Ex\rightarrow \neg \exists
% y(R_1xy\vee R_2xy))\]
%
The next conjunct expresses that either $M$ does not halt or the
representation of the computation of $M$ contains a defect. It
crucially uses non-$\Sigma$ relation symbols.  It takes the shape of
\[ \forall x \, (Dx \rightarrow \exists y \, (Nxy \wedge \psi xy)) \]
where $\psi xy$ is the following disjunction which ensures that there
is a concrete defect ($D_p^+,D_p^-,D_p^=$) here or some defect ($D$)
in some successor state:
\[ \begin{array}{l}
D(y) \vee{} \\[2mm]
\displaystyle \bigvee_{I_i=+(p,q_j)} (q_ix \wedge q_jy \wedge (D^+_pxy
\vee D^=_{1-p}xy))\vee{}\\[6mm]
\displaystyle \bigvee_{I_i=-(p,q_j,q_k)} (q_ix \wedge q_ky \wedge
(D^-_pxy \vee D^=_{1-p}xy))\vee{} \\[6mm]
\displaystyle \bigvee_{I_i=-(p,q_j,q_k)} (q_ix \wedge q_jy \wedge
(D^=_pxy \vee D^=_{1-p}xy))
%  \\[6mm] \wedge \, \forall x \forall y \, (D^+_pxy \rightarrow
%  %\\[1mm] %  \hspace*{6mm} (\neg \exists z \, R_pyz \vee (\neg
%  \exists z \, R_pxz \wedge \exists z \, (R_pyz \wedge \exists x
%  R_pzx))\\[2mm] \hspace*{29mm}\vee \, \exists z (H_{1}^{+}xyz \wedge
%  R_{p}xz \wedge \exists x (H_{2}^{+}xzy \wedge R_pyx \wedge
%  D^+_{p}zx)).
\end{array} \]
Finally, using the ternary symbols we make sure that the defects are
realized, for example, by taking:
\[ \begin{array}{l}
\forall x \forall y \, \big(D^+_pxy \rightarrow \\[1mm]
\hspace*{2mm} (\neg \exists z \, R_pyz \vee (\neg \exists z \, R_pxz
\wedge \exists z \, (R_pyz \wedge \exists x R_pzx))\vee{}\\[2mm]
\hspace*{3mm} \exists z (H_{1}^{+}xyz \wedge R_{p}xz \wedge \exists x
(H_{2}^{+}xzy \wedge R_pyx \wedge D^+_{p}zx)))\big).
\end{array} \]
Similar conjuncts implement the desired behaviour of $D^=_p$ and
$D^-_p$; since they are constructed analogously to the last three
lines above (but using guards $H^{-}_j$ and $H^{=}_j$), details are
omitted. 

Finally, we define a database $\Dmc$ by taking 
\[\Dmc = \{S(a),D(a),S(b)\}.\]
Lemmas~\ref{lem:gfundec1} and~\ref{lem:gfundec2} below establish
correctness of the reduction and thus Theorem~\ref{thm:main-gf}. 
\begin{lemma}\label{lem:gfundec1}
  If $M$ halts, then there is an $\ALC(\Sigma)$ concept that
  non-projectively separates $(\Kmc,\{a\},\{b\})$.
% 
% 
%   \begin{enumerate}
% 
%     \item If $M$ halts, then $(\Kmc,\{a\},\{b\})$ is
%     non-projectively separable by an $\ALC(\Sigma)$ concept.
% 
% %     \item If \Kmc is projectively GF$(\Sigma)$-separable, then $M$
% halts. 
% 
%   \end{enumerate}
% 
\end{lemma}
%
% The proof of Lemma~\ref{lem:gfundec1} is based on the following
% straightforward model-theoretic characterization of GF$(\Sigma)$
% separability.  % \begin{lemma} \label{lem:gfchar} Let
% $(\Kmc,P,\{\vec b\})$ be an GF learning instance and $\Sigma$ a
% signature. Then the following are equivalent: % \begin{enumerate}
% 
%     \item $(\Kmc,P,\{\vec b\})$ has a GF$(\Sigma)$ solution;
% 
%     \item There is a model $\Amf$ of \Kmc and $k\geq 0$ such that
%     for all models $\Bmf$ of \Kmc and $\vec a\in P$, we have
%     $\Bmf,\vec a\not\sim_{\text{GF}(\Sigma)}^k \Amf,\vec b$.
% 
%   \end{enumerate} % \end{lemma}
% 
%
\begin{proof}\ The idea is that the separating $\ALC(\Sigma)$ concept
  describes the halting computation of $M$, up to
  $\ALC(\Sigma)$-bisimulations. More precisely, assume that $M$ halts.
  We define an $\ALC(\Sigma)$ concept $C$ such that $\Kmc\models\neg
  C(a)$, but $\Kmc\not\models\neg C(b)$.  Intuitively, $C$ represents
  the computation of $M$ on input $(0,0)$, that is: if the computation
  is $(q_0,n_0,m_0),\dots,(q_k,n_k,m_k)$, then there is an $N$-path of
  length $k$ (but not longer) such that any object reachable in $i
  \leq k$ steps from the beginning of the path is labeled with $q_i$,
  has an outgoing $R_0$-path of length $n_i$ and no longer outgoing
  $R_0$-path, and likewise for $R_1$ and $m_i$. In more detail,
  consider the $\Sigma$-structure $\Amf$ with
  \begin{align*}
    \text{dom}(\Amf) = \{0,\ldots,k\} \cup {} & \{ a_{j}^{i} \mid
    0<i\leq k, 0<j<n_{i}\} \cup{} \\ 
    & \{ b_{j}^{i} \mid 0<i\leq k, 0<j< m_{i}\}
  \end{align*}
  in which
  $$ \begin{array}{rcl}
    N^{\Amf} &=& \{ (i,i+1) \mid i<k\} \\[1mm]
    R_{1}^{\Amf} &=& \bigcup_{i\leq k}\{ (i,a_{1}^{i}),
    (a_{1}^{i},a_{2}^{i}),\ldots,(a_{n_{i}-2}^{i},a_{n_{i}-1}^{i})\}
    \\[1mm]
    R_{2}^{\Amf} &=& \bigcup_{i\leq k}\{ (i,b_{1}^{i}),
    (b_{1}^{i},b_{2}^{i}),\ldots,(b_{m_{i}-2}^{i},b_{m_{i}-1}^{i})\}
    \\[1mm]
    S^{\Amf} &=& \{0\} \\[1mm]
	%
% 	E^{\Amf} &=& \{a^i_{n_i-1},b^i_{m_i-1}\mid i\leq k\}\\[1mm]
	%
    q^{\Amf} &=& \{ i \mid q_{i}=q\} \text{ for any } q\in Q.
  \end{array} $$
  Then let $C$ be the $\ALC(\Sigma)$ concept that describes $\Amf$
  from the point of $0$ up to $\ALC(\Sigma)$-bisimulations. Clearly,
  $\Kmc\cup\{C(b)\}$ is satisfiable. However, $\Kmc\cup\{C(a)\}$ is
  unsatisfiable since the enforced computation does not contain a
  defect and cannot be extended to have one. In particular, there are
  no $N$-paths of length $>k$ in any model of $\Kmc\cup\{C(a)\}$ and
  there are no defects in register updates in any model of
  $\Kmc\cup\{C(a)\}$.
% 
%   \smallskip Conversely, for Point~2, assume that
%   $(\Kmc,\{a\},\{b\})$ is projectively GF$(\Sigma)$-separable by a
%   GF-formula $\psi(x)$ of size $k$. Thus, there is a model $\Amf$ of
%   \Kmc such that for all models $\Bmf$ of $\Kmc$ we have
%   $\Bmf,a\not\sim_{\text{GF}(\Sigma)}^k\Amf,b$ and thus also
%   $\Bmf,a\not\sim_{\text{openGF}(\Sigma)}^k\Amf,b$. We claim that
%   the substructure of $\Amf$ starting at $b$ up to depth $k$ encodes
%   a halting computation of $M$. This can be readily proved by
%   induction on the distance from $b$ and using the propagation of
%   the defects in $\varphi$.  
% 
\end{proof}

The following lemma implies that if $M$ does not halt, then
$(\Kmc,\{a\},\{b\})$ is neither projectively $\Lmc(\Sigma)$-separable
nor non-projectively $\Lmc(\Sigma)$-separable for $\Lmc=\text{GF}$ and
in fact for every logic \Lmc between GF and FO.

\begin{lemma} \label{lem:gfundec2}
  If $M$ does not halt, then for every model $\Amf$ of \Kmc, there is
  a model $\Bmf$ of \Kmc such that $(\Amf,b^\Amf)$ is
  $\Gamma$-ismorphic to $(\Bmf,a^\Bmf)$ where $\Gamma$ consists of all
  symbols except $\text{sig}(\Omc)\setminus \Sigma$.
\end{lemma}

\noindent \begin{proof}
  Let $\Amf$ be a model of \Kmc. We obtain $\Bmf$ from \Amf by
  re-interpreting $a^\Bmf=b^\Amf$ and inductively defining the
  extensions of the symbols from 
  \[\text{sig}(\Omc)\setminus\Sigma =
  \{D,D^+_p,D^-_p,D^=_p,H_1^+,H_2^+,H_1^-,H_2^-,H_1^=,H_2^=\}.\]
  We start with $D^\Bmf = \{a^\Bmf\}$ and $X^\Bmf=\emptyset$ for all
  other symbols $X$ from $\text{sig}(\Omc)\setminus \Sigma$. Then,
  whenever $d\in D^\Bmf$ we distinguish two cases: 
  \begin{itemize}

    \item If there is an $N$-successor $e$ of $d$ such that the
      counters below $d$ and $e$ are not correctly updated with
      respect to the states at $d,e$, set the extensions of the
      symbols in
      $D^+_p,D^-_p,D^=_p,H_1^+,H_2^+,H_1^-,H_2^-,H_1^=,H_2^=$ so as to
      represent the defect and finish the construction of $\Bmf$.

    \item Otherwise, choose an $N$-successor $e$ of $d$ and add
      $e$ to $D^\Bmf$.

  \end{itemize}
  Note that, since $M$ does not halt, we can always find such an
  $N$-successor as in the second item. 
\end{proof}

% 
% We leave the case of GF$^2$-separability with and without signature
% for future work (we conjecture, however, decidability).  Instead, we
% consider another language for which conservative extensions are
% undecidable, the DL $\mathcal{ALCFIO}$. $\mathcal{ALCFIO}$ is the
% extension of \ALCI with \emph{nominals} and \emph{functionality
% assertions}. Nominals are concepts of the form $\{o\}$ and can be
% thought of as constant symbols in the context of FO. Indeed, the
% translation of the DL concept $\{o\}$ is defined by
% $\{o\}^\dagger=(x=o)$. Functionality assertions are concept
% inclusions of the form $\top\sqsubseteq (\leqslant 1\ r)$, $r$ a role,
% which demand that the role $r$ is interpreted as a function.
% 
% A direct reduction from conservative extensions as for \ALCI seems
% impossible because the relativization cannot be applied to constants. Thus,
% we adapt the strategy of the proof for undecidability of conservative
% extensions~\cite{DBLP:conf/ijcai/LutzWW07} to show the following
% theorem.

\bigskip Let us now look at $\mathcal{ALCFIO}$.

\begin{theorem}\label{thm:alcfio-main}
  Projective and non-projective \Lmc-separability with
  signature are undecidable for every logic \Lmc that contains
  $\mathcal{ALCFIO}$.
\end{theorem}
The
proof is by a reduction of the following undecidable tiling problem.
%
% \vspace*{-\medskipamount}
\begin{definition}
  A \emph{tiling system} $S=(T,H,V,R,L,T,B)$ consists of a finite set
  $T$ of \emph{tiles}, horizontal and vertical \emph{matching
    relations} $H,V \subseteq T \times T$, and sets $R,L,T,B \subseteq
  T$ of \emph{right} tiles, \emph{left} tiles, \emph{top} tiles, and
  \emph{bottom} tiles. A \emph{solution} to $S$ is a triple $(n,m,\tau)$
  where $n,m \in \mathbb{N}$ and $\tau: \{0,\ldots,n\} \times \{0,\ldots,m\} \rightarrow T$ such
  that the following hold:
\begin{enumerate}

\item $(\tau(i,j),\tau(i+1,j)) \in H$, for all $i<n$ and $j \leq m$;

\item $(\tau(i,j),\tau(i,j+1)) \in V$, for all $i\leq n$ and $j<m$;

\item $\tau(0,j) \in L$ and $\tau(n,j) \in R$, for all $j \leq m$;

\item $\tau(i,0) \in B$ and $\tau(i,m) \in T$, for all $i \leq n$.

\end{enumerate}
\vspace*{-\medskipamount}
\end{definition}
We show how to convert a tiling system $S$ into a labeled
$\mathcal{ALCFIO}$-KB $(\Kmc,P,N)$ and signature $\Sigma$ such that
$S$ has a solution iff $(\Kmc,P,N)$ is
$\mathcal{ALCFIO}(\Sigma)$-separable iff $(\Kmc,P,N)$ is projectively
$\mathcal{ALCFIO}(\Sigma)$-separable.

Let $S=(T,H,V,R,L,T,B)$ be a tiling system. Define an ontology \Omc
that consists of the following statements:
\begin{itemize}

\item The roles $r_x$, $r_y$, and their inverses are functional:
  $$
  \top \sqsubseteq (\leqslant 1\ r), \text{ for } r \in
  \{r_x,r_y,r_x^-,r_y^-\}
  $$

\item Every grid node is labeled with exactly
  one tile and the matching conditions are satisfied:
  $$
  \begin{array}{rcl}
  \top &\sqsubseteq& \bigsqcup_{t\in T}(t \sqcap \bigsqcap_{t' \in T,\; t'\not=t} \neg t') \\[4mm]
  \top &\sqsubseteq& \bigsqcap_{t\in T}(t \rightarrow (\bigsqcup_{(t,t') \in H} \forall r_x . t' \sqcap \bigsqcup_{(t,t') \in V} \forall r_y . t'))
  \end{array}
  $$

\item The concepts \mn{left}, \mn{right}, \mn{top}, \mn{bottom} mark the borders of
  the grid in the expected way:
$$
\begin{array}{rcl}
  \mn{right} &\sqsubseteq& \neg \exists r_{x}.\top \sqcap \forall r_{y}.\mn{right} \sqcap \forall r_{y}^{-1}.\mn{right} \\
  \neg \mn{right} &\sqsubseteq& \exists r_{x}.\top \\
\end{array}
$$
and similarly for \mn{left}, \mn{top}, and \mn{bottom}.

\item The individual name $o$ marks the origin:
  $$
    \{o\} \sqsubseteq \mn{left} \sqcap \mn{bottom}.
  $$

\item there is no infinite outgoing $r_x$/$r_y$-path starting at
  $o$ and grid cells close in the part of models 
  reachable from $o$:
  $$
\begin{array}{rcl}
Q &\sqsubseteq& \exists r_{x}. Q \sqcup \exists r_{y}.Q \sqcup
(\exists r_{x}.\exists r_{y}.P \sqcap \exists r_{y}. \exists r_{x}.\neg P) \\
A_1 \sqcap A_2 &\sqsubseteq& \exists u .(\{o\} \sqcap Q)
\end{array}
  $$
\end{itemize}
The final item deserves some further explanation. It is to be read as
follows: the stated properties hold in a model \Amf whenever \Amf can
\emph{not} be extended to a model of the upper CI that makes true $Q$
at $o$. In conjunction with the database, the second CI is a switch
that will allow us to sometimes require that $Q$ is made true at $o$.

Set $\Sigma = T\cup \{ r_x,r_y, \mn{left}, \mn{right}, \mn{top}, \mn{bottom}
\}$ and consider the labeled KB $(\Kmc,\{a\},\{b\})$ where
$\Kmc=(\Omc,\Dmc)$ with $\Dmc = \{ A_1(a), Y(b) \}$ with $Y$
a fresh (dummy) concept name.

\begin{lemma}
  If  $S$ has a solution, then there is an $\mathcal{ALCIO}(\Sigma)$
  concept that non-projectively separates $(\Kmc,\{a\},\{b\})$.
\end{lemma}
\begin{proof} We design the $\mathcal{ALCI}(\Sigma)$ concept $C=\neg
  D$ so that any model of $D$ and \Omc, even without the CIs from the
  last item, includes a properly tiled $n \times m$-grid with lower
  left corner $o$.
  
  For every word $w \in \{r_x,r_y\}^*$, denote by $\overleftarrow{w}$
  the word that is obtained by reversing $w$ and then adding
  $\cdot^{-}$ to each symbol. Let $|w|_r$ denote the number of
  occurrences of the symbol $r$ in $w$.  Now, $D=A_2 \sqcap
  \exists u . E$ where   $E$ is the conjunction
  of
$$
  \{o\} \sqcap \forall r_x^n . \mn{right} \sqcap \forall r_y^m . \mn{top}
$$
and for every $w \in \{r_x,r_y\}^*$ such that $|w|_{r_x} < n$ and
$|w|_{r_y} < m$, the concept 
$$ 
\exists (w \cdot r_x r_y r_x^- r_y^- \cdot \overleftarrow{w}) . \{ o
\},
$$
where $\exists w . F$ abbreviates
$\exists r_1 . \cdots \exists r_k .  F$ if $w=r_1 \cdots r_k$.  It is
readily checked that $E$ (and thus $D$) indeed enforces a properly
tiled grid as announced. Then, due to the CIs in the last item,
$\Kmc\cup\{D(a)\}$ is unsatisfiable: Any model $\Amf$ has to satisfy 
$a^\Amf \in A_1^\Amf$ since $A_1(a)\in \Dmc$ 
and $a^\Amf\in A_2^\Amf$, due to the assertion $D(a)$. Hence the last
CIs become `active', which is in conflict to the fact that
the model enforced by $D$ contains neither an infinite
$r_x$/$r_y$-path nor a non-closing grid-cell. 
Thus, $\Kmc\models C(a)$ as required. 

Now for $\Kmc \not\models C(b)$. We find a model \Amf of \Kmc with
$b^\Amf \in D^\Amf$ since all CIs in \Omc except this from the last
item are satisfied by the grid enforced by $D$ and the CIs in the last
item can be made `inactive' by making $A_1$ false at
$b^\Amf$. 
\end{proof}
The following lemma implies that if $S$ has no solution, then
$(\Kmc,\{a\},\{b\})$ is neither projectively $\Lmc(\Sigma)$-separable
nor non-projectively $\Lmc(\Sigma)$-separable for
$\Lmc=\mathcal{ALCIO}$ and in fact for every logic \Lmc between
\ALCIO and $\text{FO}$.
\begin{lemma}
  If $S$ has no solution, then for every model \Amf of \Kmc, there is
  a model \Bmf of \Kmc such that $(\Amf,b^\Amf)$ is
  $\Gamma$-isomorphic to $(\Bmf,a^\Amf)$ where $\Gamma$
  consists of all symbols except $\{ u, A_1,Q,  P\}$.
\end{lemma}
\begin{proof}
(sketch)  
If $b^\Imc \notin A_2^\Amf$, then we can simply obtain \Bmf
from \Amf by switching $a^\Amf$ and $b^\Amf$ and making $A_1$ true at
$a^\Amf$. If $b^\Amf \in A_2^\Amf$, then after switching we
additionally have to re-interpret $Q$, $P$, and $u$ in a suitable way. But
$S$ has no solution and thus when following $r_x$/$r_y$-paths from
$o$ in \Imc, we must either encounter an infinite such path or a
non-closing grid cell as otherwise we can extract from \Imc a solution
for $S$. Thus we can re-interpret $Q$, $P$, and $u$  as required.
\end{proof}

\section{Strong Separability with Signature}
\label{sec:dfstrong}

We introduce strong separability of labeled KBs. The crucial
difference to weak separability is that the negation of the separating
formula must be entailed at all negative examples.

\begin{definition}
   Let $(\Kmc,P,N)$ be a labeled FO-KB and $\Sigma\subseteq \text{sig}(\Kmc)$ a signature.
    An FO-formula $\varphi(\vec{x})$ \emph{strongly $\Sigma$-separates}
	$(\Kmc,P,N)$ if $\text{sig}(\Kmc)\cap \text{sig}(\varphi)\subseteq \Sigma$ and
	\begin{enumerate}
		
		\item 
		$\Kmc\models \varphi(\vec{a})$ for all $\vec{a}\in P$ and 
		
		\item $\Kmc\models \neg\varphi(\vec{a})$ for all $\vec{a}\in N$.
		
	\end{enumerate}
	Let $\Lmc_S$ be a fragment of FO. We say that $(\Kmc,P,N)$ is
	\emph{strongly projectively $\Lmc_S(\Sigma)$-separable} if there 
	exists an
	$\Lmc_S(\Sigma)$-formula $\varphi(\vec{x})$ that strongly separates
	$(\Kmc,P,N)$ and \emph{non-projectively strongly
	$\Lmc_S(\Sigma)$-separable} if there is such a
	$\varphi(\vec{x})$ with $\text{sig}(\varphi) \subseteq \Sigma$.
\end{definition}
In contrast to weak separability, any formula $\varphi$ that strongly
separates a labeled KB $(\Kmc,P,N)$ and uses helper symbols $R$ that are
not in $\Sigma$ can easily be transformed into a strongly separating
formula that uses only symbols from $\Sigma$: simply replace any
such $R$ by a
relation symbol $R'$ of the same arity that is in $\Sigma$. Then, if
$\varphi$ strongly separates $(\Kmc,P,N)$, so does the resulting
formula $\varphi'$.  If no relation symbol of the same arity as $R$
occurs in $\Sigma$ one can alternatively replace relevant subformulas
by $\top$ or $\bot$.  In what follows, we thus only consider
non-projective strong separability and simply speak of strong
separability.

Note that for languages $\Lmc_{S}$ closed under conjunction and
disjunction a labeled KB $(\Kmc,P,N)$ is strongly
$\Lmc_{S}(\Sigma)$-separable iff every
$(\Kmc,\{\vec{a}\},\{\vec{b}\})$ with $\vec{a}\in P$ and
$\vec{b}\in N$ is strongly $\Lmc_{S}(\Sigma)$-separable.  In fact, if
$\varphi_{\vec{a},\vec{b}}$ strongly separates
$(\Kmc,\{\vec{a}\},\{\vec{b}\})$ for $\vec{a}\in P$ and
$\vec{b}\in N$, then
$\bigvee_{\vec{a}\in P}\bigwedge_{\vec{b}\in
  N}\varphi_{\vec{a},\vec{b}}$ strongly separates $(\Kmc,P,N)$.
Without loss of generality, we may thus work with labeled KBs with
singleton sets of positive and negative examples.

Each choice of an ontology language $\Lmc$ and a separation language
$\Lmc_{S}$ thus gives rise to a (single) strong separability problem
that we refer to as \emph{strong $(\Lmc,\Lmc_S)$-separability},
defined in the expected way:

\begin{center}
	\fbox{\begin{tabular}{ll}
			\small{PROBLEM} : & strong $(\Lmc,\Lmc_{S})$ separability with signature\\
			{\small INPUT} : &  labeled $\Lmc$-KB $(\Kmc,P,N)$ and signature $\Sigma \subseteq \text{sig}(\Kmc)$ \\
			{\small QUESTION} : & Is $(\Kmc,P,N)$ strongly $\LmcO_{S}(\Sigma)$-separable?  
		\end{tabular}}
	\end{center}
If $\Lmc=\Lmc_{S}$, then we simply speak of strong $\Lmc$-separability. The study of strong separability is very closely linked to the study of interpolants and the Craig interpolation property. Given formulas $\varphi(\vec x),\psi(\vec x)$ and a fragment $\Lmc$ of FO, we say that an $\Lmc$-formula $\chi(\vec x)$ is an \emph{$\Lmc$-interpolant of $\varphi,\psi$} if $\varphi(\vec x)\models
\chi(\vec x)$, $\chi(\vec x) \models \psi(\vec x)$, and $\text{sig}(\chi)
\subseteq \text{sig}(\varphi)\cap \text{sig}(\psi)$.
We say that \emph{$\Lmc$ has the CIP} if for any $\Lmc$-formulas $\varphi(\vec x),\psi(\vec x)$ such that $\varphi(\vec x)\models \psi(\vec x)$ there exists an $\Lmc$-interpolant of $\varphi,\psi$.
FO has the CIP, and so does GNF~\cite{DBLP:journals/tocl/BenediktCB16,DBLP:journals/jsyml/BaranyBC18}, at least if one admits non-shared constants in the interpolant. On the other hand, GF
does not have the CIP~\cite{DBLP:journals/sLogica/HooglandM02}. The link between the interpolants, the CIP, and strong separability is easy to see: assume
a labeled FO-KB $(\Kmc,\{\vec a\},\{\vec b\})$ with
$\Kmc=(\Omc,\Dmc)$ and a signature $\Sigma \subseteq
\text{sig}(\Kmc)$ are given. Obtain $\Kmc_{\Sigma,\vec{a}}$ and $\Kmc_{\Sigma,\vec{b}}$ from
$\Kmc$ by 
\begin{itemize}
	
	\item replacing all non-$\Sigma$-relation symbols $R$ in $\Kmc$ by
	fresh symbols $R^{\vec{a}}$ and $R^{\vec{b}}$, respectively;
	
	\item replacing all constant symbols $c$ by 
	fresh variables $x_{c,\vec{a}}$ and $x_{c,\vec{b}}$ which are distinct, except that $\vec a$ and $\vec b$ are replaced by the same tuple $\vec x$ in $\Kmc_{\Sigma,\vec{a}}$ and $\Kmc_{\Sigma,\vec{b}}$, respectively. 
%	We refer to the
%	variable replacing $c$ in $\Kmc_{\Sigma,\vec{a}}$ and %$\Kmc_{\Sigma,\vec{b}}$ by $x_{c,\vec{a}}$ and $x_{c,\vec{b}}$, %respectively.
\end{itemize}	
Then let $\varphi_{\Sigma,\vec{a}}(\vec x) = \exists \vec z (\bigwedge
\Kmc_{\Sigma,\vec{a}})$, where $\vec z$ is the sequence of free variables in $\Kmc_{\Sigma,\vec{a}}$ without the variables in
$\vec x$ and $(\bigwedge \Kmc_{\Sigma,\vec{a}})$ is the conjunction of all formulas
in $\Kmc_{\Sigma,\vec{a}}$. $\varphi_{\Sigma,\vec{b}}(\vec x)$ is defined in the same way, with $\vec{a}$ replaced by $\vec{b}$. The following lemma is a direct consequence of the construction.
\begin{lemma}\label{lem:int}
	Let $\Lmc$ be a fragment of FO. Then the following conditions are equivalent for any formula $\varphi$ in $\Lmc$:
	\begin{enumerate}
		\item $\varphi$ strongly $\Lmc(\Sigma)$-separates $(\Kmc,\{\vec{a}\},\{\vec{b}\})$;
		\item $\varphi$ is an $\Lmc$-interpolant for $\varphi_{\Sigma,\vec{a}}(\vec x),\neg \varphi_{\Sigma,\vec{b}}(\vec{x})$.
	\end{enumerate}	
\end{lemma}
Thus, the problem whether a labeled KB $(\Kmc,P,N)$ is strongly $\Lmc(S)$-separable and the computation of a strongly separating formula can be equivalently formulated as an interpolant existence problem. As FO has the CIP, we obtain the following characterization
of the existence of strongly FO$(\Sigma)$-separating formulas.
\begin{theorem}
	The following conditions are equivalent for any labeled FO-KB $(\Kmc,\{\vec{a}\},\{\vec{b}\})$ and signature $\Sigma\subseteq \text{sig}(\Kmc)$:
		\begin{enumerate}
		\item $(\Kmc,\{\vec{a}\},\{\vec{b}\})$ is strongly FO$(\Sigma)$-separable;
		\item $\varphi_{\Sigma,\vec{a}}(\vec x) \models \neg \varphi_{\Sigma,\vec{b}}(\vec{x})$.
	\end{enumerate}
\end{theorem}
For fragments $\Lmc$ of FO such as $\mathcal{ALCI}$, GF, and GNF, 
Lemma~\ref{lem:int} has to be applied with some care, as one has to ensure that the formulas $\varphi_{\Sigma,\vec{a}}(\vec x),\neg \varphi_{\Sigma,\vec{b}}(\vec x)$ are still within $\Lmc$. This
will be discussed in the next two sections. 
 
\section{Strong Separability in $\mathcal{ALCI}$}
We first compare the strong separating power of $\ALCI$ with signature restrictions to the strong separating power of FO 
with signature restrictions and show that they differ. This is 
in contrast to strong separability without signature restrictions. 
We then show that strong $\mathcal{ALCI}$-separability with signature restrictions is
\TwoExpTime-complete, thus one exponential harder than
strong $\mathcal{ALCI}$-separability without signature restrictions. 
Observe that we cannot apply the CIP of $\mathcal{ALCI}$~\cite{TenEtAl13} to investigate strong separability 
for $\ALCI$-KBs as one cannot encode the atomic formulas of the database in $\mathcal{ALCI}$.\footnote{One could instead move to the extension
$\mathcal{ALCIO}$ of $\mathcal{ALCI}$ with nominals. This language,
however, does not have the CIP~\cite{TenEtAl13}. Recently, interpolant existence in $\mathcal{ALCIO}$ has been investigated in~\cite{All2020}, and
the results could be applied here. The following direct approach 
is of independent value, however.}

In~\cite{KR}, strong separability is studied without signature restrictions. It turned that a labeled $\mathcal{ALCI}$-KB is strongly $\mathcal{ALCI}$-separable without signature restrictions iff it is strongly FO-separable without signature restrictions. 
Unfortunately, this is not the case with signature restrictions.
A simple counterexample is given in the following example.
\begin{example}
	Let $\Dmc = \{R(a,a),A(b)\}$ and $\Omc = \{ A \sqsubseteq \forall R.\neg A\}$. Let $\Kmc=(\Omc,\Dmc)$ and $\Sigma=\{R\}$. Then $R(x,x)$ strongly separates $(\Kmc,\{a\},\{b\})$ and thus $(\Kmc,\{a\},\{b\})$ is strongly FO$(\Sigma)$-separable. The characterization below immediately implies that $(\Kmc,\{a\},\{b\})$ is not strongly $\mathcal{ALCI}(\Sigma)$-separable.	
\end{example}	
We now show that strong $\mathcal{ALCI}$-separability with signature restrictions is
\TwoExpTime-complete. To this end we first give a
model-theoretic characterization of strong
$\ALCI$-separability using $\mathcal{ALCI}$-bisimulations. 	

\begin{theorem}\label{thm:alcistrongcrit}
  Let $(\Kmc,\{a\},\{b\})$ be an $\ALCI$-KB and $\Sigma\subseteq
  \text{sig}(\Kmc)$ a signature.  Then the following conditions are
  equivalent: 
\begin{enumerate} 

  \item $(\Kmc,\{a\},\{b\})$ is strongly $\ALCI(\Sigma)$-separable; 

  \item There are no models $\Amf$ and $\Bmf$ of $\Kmc$ such that
    $\Amf,a^{\Amf}\sim_{\ALCI,\Sigma} \Bmf,b^{\Bmf}$.

\end{enumerate} 

\end{theorem}
The proof is straightforward
using Lemma~\ref{lem:equivalence}.  By working with isomorphic copies
of the database $\Dmc$ it thus suffices to show the following result.

\begin{lemma}
  Let $(\Kmc,\{a\},\{b\})$ be a labeled $\ALCI$-KB with
  $\Kmc=(\Omc,\Dmc)$ such that $a,b$ are in distinct maximal connected
  components of $\Dmc$. Then the problem to decide whether there
  exists a model \Amf of $\Kmc$ such that
  $\Amf,a^\Amf\sim_{\ALCI,\Sigma}\Amf,b^\Amf$ is \TwoExpTime-complete.
\end{lemma}

Let $\Kmc=(\Omc,\Dmc)$ and let $\Sigma\subseteq \text{sig}(\Kmc)$ be a signature. We start with proving the upper bound and
use the notion of \Kmc-types as introduced in Section~\ref{sec:alciweak}.
Let $R$ be a role. We say that $\Kmc$-types $t_{1}$ and $t_{2}$ are \emph{$R$-coherent} if there exists a
model $\Amf$ of $\Omc$ and nodes $d_{1}$ and $d_{2}$ realizing $t_{1}$
and $t_{2}$, respectively, such that $(d_{1},d_{2})\in R^{\Amf}$. We
write $t_1\rightsquigarrow_R t_2$ in this case. 

\begin{definition}[$(\Omc,\Sigma)$-amalgamable]
  A set $\Phi$ of $\Kmc$-types is \emph{$(\Omc,\Sigma)$-amalga-mable} if there exist
  models $\Amf_{t}$ of $\Omc$ for $t\in \Phi$ with elements $d_{t}$
  realizing $t$ in $\Amf_{t}$ such that all $\Amf_{t},d_{t}$ with
  $t\in \Phi$ are $\ALCI(\Sigma)$-bisimilar. 
\end{definition}

\begin{lemma}\label{lem:amalgable}
	The set of all $(\Omc,\Sigma)$-amalgamable sets of $\Kmc$-types can be computed
	in double exponential time.
\end{lemma}

We devise an elimination procedure as follows. Start with $M_0$ the
set of all sets of \Kmc-types. Given a set $M_i$ of sets of types, we
obtain $M_{i+1}$ by eliminating all $\Phi=\{t_1,\ldots,t_n\}$ from
$M_i$ which do not satisfy the following conditions: 
\begin{enumerate}
	%[left=0pt,label=(\textbf{E\arabic*})]

  \item for every $A\in \Sigma$, we have $A\in t_i$ iff $A\in t_j$, for all
    $t_i,t_j\in \Phi$;

  \item for every $t_i$, every $\Sigma$-role $R$, and
    every $\exists R.C\in t_i$ there are \Kmc-types
    $t_1',\ldots,t_n'$ such that $C\in t_i'$ and
    $t_j\rightsquigarrow_R t_j'$, for all $j$, and
    $\{t_1',\ldots,t_n'\}\in M_i$.

\end{enumerate}
Let $M^*$ be where the sequence $M_0,M_1,\ldots$ stabilizes. 

\medskip\noindent\textit{Claim}. $\Phi\in M^*$ iff $\Phi$ is
$(\Omc,\Sigma)$-amalgamable.

\medskip\noindent\textit{Proof of the Claim}. For the
``if''-direction, suppose that $\Phi=\{t_1,\ldots,t_n\}$ is
$(\Omc,\Sigma)$-amalgamable. We can fix (disjoint) models
$\Amf_{t_1},\ldots,\Amf_{t_n}$ of \Omc realizing types $t_i$ at
$d_{t_i}$. Let $\Amf$ denote the union of
$\Amf_{t_1},\ldots,\Amf_{t_n}$ and let $S$ be the set of all pairs
$(d,e)$ which are $\Sigma$-bisimilar in $\Amf$. Recall that $S$ is an
equivalence relation. By assumption, we have $(d_{t_i},d_{t_j})\in
S$, for all $i,j$. It can be verified that the set $N$ defined by
\[N = \{ \{\text{tp}_{\Kmc}(\Amf,d) \mid (d,e)\in S\}\mid e\in \text{dom}(\Amf)\}\]
is contained in all $M_i$ and thus in $M^*$.

\smallskip
For ``only if'', let $\Phi=\{t_1,\ldots,t_n\}\in M^*$. We
inductively construct a domain $\Delta$, a map $\pi$ of
the domain $\Delta$ to \Kmc-types, and an equivalence relation
$S$. During the construction, we preserve the invariant
\begin{itemize}

  \item[$(\ast)$] $\pi(D)=\{\pi(d)\mid d\in D\}\in M^*$, for every
    $D\in S$.

\end{itemize}
	
	For the construction, start with setting
	\begin{itemize}
		
		\item $\Delta_0=\{d_{1},\ldots,d_{n}\}$, $\pi(d_i)=t_i$,
		for all $i$, and $S=\{\Delta_0\}$.
		
	\end{itemize}
	Obviously, the invariant is satisfied. To obtain $\Delta_{i+1}$
	from $\Delta_i$, choose some $D=\{e_1,\ldots,e_m\}\in S$, some
	$\exists R.C\in \pi(d_i)$ for some $\Sigma$-role $R$. By the
	invariant, we have $\{t_1,\ldots,t_k\}=\pi(D)\in M^*$. Let
	$t_1',\ldots,t_k'$ be the types that exist due to~(\textbf{E2}). Now, add
	fresh elements $e_1Re_1',\ldots,e_mRe_m'$ to $\Delta_i$, set
	$\pi(e_iRe_{i}')=\pi(e_i)'$, for all $i$, and add
	$\{e_1Re_1',\ldots,e_mRe_m'\}$ to $S$. By construction, the
	invariant~$(\ast)$ is preserved. 
	
	Now define a
	structure \Amf by taking: 
	\begin{align*}
	\text{dom}(\Amf)& = \bigcup_{i\geq 0}\Delta_i \\
	A^\Amf& = \{e\mid A\in \pi(e)\} \\
	r^\Amf & = \{(d,dRe)\mid dRe\in \text{dom}(\Amf)\} \cup{}\\
	&\phantom{ {}={}} \{(dR^{-}e,d)\mid dR^{-}e\in \text{dom}(\Amf)\}
	\end{align*}
	By construction, $S$ is an $\ALCI(\Sigma)$-bisimulation that contains
	$(d_i,d_j)$ for all $i,j$. Since \Kmc-types are realizable by
	definition, we can extend \Amf to a model $\Amf^*$ of $\Omc$ by
	adding non-$\Sigma$-subtrees whenever they are needed. 
	
	It follows that $\Phi$ is $(\Omc,\Sigma)$-amalgamable.
	This finishes the proof of the Claim.
	
	\smallskip It remains to discuss the running time of the algorithm.
	The initial set $M_0$ contains at most double exponentially many
	elements. Since in every round some element is removed from $M_i$,
	the stabilization is reached after $|M_0|$ rounds. It remains to
	observe that the elimination conditions~1 and~2 can be checked in
	double exponential time. 
	%
%\end{proof}

It is important to note that the proof of Lemma~\ref{lem:amalgable}
shows that, if a set $\Phi$ is $(\Omc,\Sigma)$-amalgamable, then this
is witnessed by disjoint tree-shaped models $\Amf_t$ with root $d_t$,
for each $t\in\Phi$, and $\ALCI(\Sigma)$-bisimulations $S$ which ``never visit
the roots again'', that is, if $(d_t,e)\in S$ or $(e,d_t)\in S$, then
$e=d_{t'}$ for some $t'$. This will be used in the proof of the
characterization below. 

Let $\Psi$ be a mapping associating with every $c\in \text{cons}(\Dmc)$
a $\Kmc$-type $t_{c}$ and a set $\Phi_{c}$ of $\Kmc$-types. We say
that $\Psi$ is \emph{$\Kmc,a,b$-satisfiable} if 
\begin{enumerate}
	
	\item there exists a model $\Amf$ of $\Kmc$ realizing $t_{c}$
	  in $c^\Amf$ for $c\in \text{cons}(\Dmc)$;
	
	\item $\Phi_{c}\cup \{t_{c}\}$ is  $(\Omc,\Sigma)$-amalgamable, for
	all $c\in \text{cons}(\Dmc)$;
	
	\item $\Phi_{a}\cup \Phi_{b}\cup \{t_{a},t_{b}\}$ is $(\Omc,\Sigma)$-amalgamable;
	
	\item If $R(d,e)\in \Omc$, for some $\Sigma$-role $R$, and $t\in \Phi_{d}$, then there exists
	$t'\in \Phi_{e}$ such that $t\rightsquigarrow_R t'$;
	
	% \item If $r(c,d)\in \Amc$ and $t_{c}'\in \Phi_{c}$, then there exists $t_{d}'\in \Phi_{d}$ such that $t_{c}',t_{d}'$ are $r$-coherent;
	% 
	% %\item If $r(a,c),\in \Amc$ and $t_{a}'\in \Phi_{a}$, then there exists $t_{c}'\in \Phi_{c}$ such that $t_{a'}$ and $t_{c'}$ are $r$-coherent;
	% \item If $r(a,c)\in \Amc$ and $t_{b}'\in \Phi_{b}$, then there exists $t_{c}'\in \Phi_{c}$ such that $t_{b'}$ and $t_{c'}$ are $r$-coherent;
	% \item If $r(b,c)\in \Amc$ and $t_{a}'\in \Phi_{a}$, then there exists $t_{c}'\in \Phi_{c}$ such that $t_{a'}$ and $t_{c'}$ are $r$-coherent;
	% %\item If $r(b,c),\in \Amc$ and $t_{b}'\in \Phi_{b}$, then there are $t_{c}'\in \Phi_{c}$ such that $t_{b'}$ and $t_{c'}$ are $r$-coherent;
\end{enumerate}
\begin{lemma}\label{lem:charbisi}
	The following conditions are equivalent:
	\begin{itemize}
		
		\item There exists a model \Amf of $\Kmc$ such that
		$\Amf,a^\Amf\sim_{\ALCI,\Sigma}\Amf,b^\Amf$.
		
		\item There exists $\Psi$ that is $\Kmc,a,b$-satisfiable.
		
	\end{itemize}
	
\end{lemma}
For ``only if'', let \Amf be a model of $\Kmc$ such that
$(a^\Amf,b^\Amf)\in S$.  Let $t_{c}$ be the type realized in $c^\Amf$
for $c\in \text{cons}(\Dmc)$ and let $\Phi_{c}$ be the set of all
$\Kmc$-types realized in nodes that are $\ALCI(\Sigma)$-bisimilar in
\Amf to $c^\Amf$, that is, 
\[\Phi_c = \{ \text{tp}_{\Kmc}(\Amf,d) \mid
\Amf,c^\Amf\sim_{\ALCI,\Sigma}\Amf,d^\Amf\}.\]

It is easy to see that the resulting $\Psi$ is as required.

\medskip Conversely, assume $\Psi$ is given. Due to Condition~1, we
can fix a model \Bmf of \Kmc realizing $t_c$ in $c^\Bmf$ for $c\in
\text{cons}(\Dmc)$. Moreover, due to Condition~2, we can fix for any
$c\in \text{cons}(\Dmc)$ and $t\in \Phi_{c}\cup \{t_{c}\}$ tree-shaped
models $\Amf_{c,t}$ of \Omc with root $d_{c,t}$ such that all pointed
structures in 
\[\{ \Amf_{c,t},d_{c,t} \mid t\in \Phi_{c}\cup\{t_{c}\}\}\]
are $\ALCI(\Sigma)$-bisimilar. By Condition~3, we can assume that
$\Amf_{a,t_a},d_{a,t_a}$ and $\Amf_{b,t_b},d_{b,t_b}$ are $\ALCI(\Sigma)$-bisimilar.

We inductively construct a model \Amf and a
$\ALCI(\Sigma)$-bisimulation $S$.  We start with the structure
$\Amf_0$ which is obtained as follows. Let $\Bmf'$ be \Bmf restricted
to the domain $\{d^\Bmf\mid d\in\text{cons}(\Dmc)\}$. Now, $\Amf_0$ is
the union of $\Bmf'$ and all structures $\Amf_{c,t_c}$ for all
$c\in\text{cons}(\Dmc)$, always identifying the root of $\Amf_{c,t_c}$
with $c^\Bmf$. Moreover, let $S_0$ be the $\ALCI(\Sigma)$-bisimulation
between $\Amf_{a,t_a}$ and $\Amf_{b,t_b}$. By the comment after the
proof of Lemma~\ref{lem:amalgable}, $(a^{\Amf_0},b^{\Amf_0})$ is the
only tuple in $S_0$ that contains $a^{\Amf_0}$ or $b^{\Amf_0}$.

Note that $\Amf_0$ is in fact a model of $\Kmc$ but $S_0$ is not yet a
bisimulation. In order to make it one, we ``chase'' the database in
both connected components preserving the following invariant (which is
obviously satisfied for $\Amf_0$, $S_0$): 
\begin{itemize}

  \item[$(\ast)$] If $(c,d^{\Amf_i})\in S_i$ or $(d^{\Amf_i},c)\in
    S_i$ for some $d\in \text{cons}(\Dmc)$, then the type $t$ realized by $c$ in
    $\Amf_i$ satisfies $t\in\Phi_d$. Moreover, the trees below
    $c$ and $d^{\Amf_i}$ are $\ALCI(\Sigma)$-bisimilar.

\end{itemize}
	In the inductive step, obtain $\Amf_{i+1},S_{i+1}$ from $\Amf_i,S_i$
	by applying one of the following rules:
	\begin{itemize}
		
		\item Choose $(c,d^{\Amf_i})\in S_i$ and $e$ such that $R(d,e)\in
		\Dmc$, and let $t=\text{tp}_{\Amf_i}(c)$ be the type of $c$ realized
		in $\Amf_{i}$. By~$(\ast)$, we know that $t\in \Phi_d$. 
		By Condition~4, we can choose $t'\in \Phi_e$ with
		$t\rightsquigarrow_R t'$. Now, add a copy of $\Amf_{e,t'}$ to
		$\Amf_i$ and make its root an $R$-successor of $c$. By
		Condition~2, there is an $\ALCI(\Sigma)$-bisimulation $S$ between
		$\Amf_{e,t'}$ and the tree $\Amf_{e,t_e}$ below $e^\Amf_i$. Set
		$S_{i+1}=S_i\cup S$.
		
		\item Choose $(d^{\Amf_i},c)\in S_i$ and $e$ such that $R(d,e)\in
		\Dmc$, and proceed analogously to the first rule. 
		
	\end{itemize}
	Let $\Amf=\bigcup\Amf_i$ and $S=\bigcup S_{i}$. 
	
	\smallskip\noindent\textit{Claim.} $\Amf$ is a model of $\Kmc$ and $S$ is
	$\ALCI(\Sigma)$-bisimulation with $(a^\Amf,b^\Amf)\in S$.
	
	\smallskip\noindent\textit{Proof of the Claim.} We have $\Amf\models\Kmc$
	since $\Amf_i\models\Kmc$, for all $i$. Moreover, $(a^\Amf,b^\Amf)\in
	S$ since $(a^\Amf,b^\Amf)\in S_0$, by definition of $S_0$. To see that
	$S$ is an $\ALCI(\Sigma)$-bisimulation, let $(d,e)\in S$. We distinguish two
	cases:
	\begin{itemize}
		
		\item None of $d,e$ is in $\{c^\Amf\mid c\in\text{cons}(\Dmc)\}$.
		Thus, $(d,e)\in S$ because $d$ and $e$ are inner nodes of some of
		the trees $\Amf_{c,t}$ that were fixed in the beginning. By
		construction, $(d,e)$ is an element of a $\ALCI(\Sigma)$-bisimulation
		$S'$ between those trees. Thus, for every $R$-successor $d'$ of $d$
		in \Amf, $R$ a $\Sigma$-role, there is an $R$-successor $e'$ of
		$e$ with $(d',e')\in S'$ and thus $(d',e')\in S$. The
		forth-condition is analgous.
		
		\item One of $d,e$ is in $\{c^\Amf\mid c\in\text{cons}(\Dmc)\}$, say
		$d=f^\Amf$. Suppose $d'$ is an $R$-successor of $d$, for some
		$\Sigma$-role $R$. We distinguish two cases:
		\begin{itemize}
			
			\item $d'$ is in the subtree $\Amf_{d,t_{d}}$ below $d$. Then
			because of~$(\ast)$, there is an $R$-successor $e'$ of
			$e$ in the tree below $e$ such that $(d',e')\in S$.
			
			\item $d'=g^\Amf$ for some $R(f,g)\in \Dmc$. Since the rules are
			applied exhaustively, there is an $R$-successor $e'$ of
			$e$ in $\Amf$ such that $(d',e)\in S$.
			
		\end{itemize}
		The forth-condition is analgous.
		
	\end{itemize}
	This finishes the proof of the Claim and, in fact, of the Lemma.
	%
%\end{proof}

We can thus use the following algorithm to decide strong
$\ALCI(\Sigma)$-separability on input $(\Kmc,P,N)$.
\begin{enumerate}
	
	\item compute the set of all $(\Omc,\Sigma)$-amalgamable sets. 
	
	\item for all $a\in P$ and $b\in P$: 
	
	\begin{enumerate}
		
		\item enumerate all possible mappings $\Psi$ consisting of $\Kmc$-types
		$t_c$ and sets of $\Kmc$-types $\Phi_c$, for every
		$c\in\text{cons}(\Dmc)$.
		
		\item if $\Psi$ is $\Kmc,a,b$-satisfiable, that is, satisfies
		Conditions~1--4 above, return ``not separable.'' 
		
	\end{enumerate}
	
	\item return ``separable''.
	
\end{enumerate}
The algorithm is correct due to Theorem~\ref{thm:alcistrongcrit} and
Lemma~\ref{lem:charbisi}. Moreover, it runs in double exponential time
since Step~1 can be executed in double exponential time, by
Lemma~\ref{lem:amalgable}, there are only double exponentially many
possible mappings $\Psi$, and $\Kmc,a,b$-satisfiability can be checked
in double exponential time: Condition~1 can be done in exponential
time, Conditions~2 and~3 are a mere lookup in the (precomputed)
amalgamable sets, and Condition~4 can be tested in double exponential
time. 

For the \TwoExpTime lower bound, we reduce the word problem for
exponentially space bounded alternating Turing machines (ATMs). We
actually use a slightly unusual ATM model which is easily seen to be
equivalent to the standard model.  

An \emph{alternating Turing machine (ATM)} is a tuple
$M=(Q,\Theta,\Gamma,q_0,\Delta)$ where $Q=Q_{\exists}\uplus
Q_{\forall}$ %\uplus\{q_a,q_\circlearrowleft^1,q_\circlearrowleft^2\}$
is the set of states that consists of \emph{existential states}
in~$Q_{\exists}$ and \emph{universal states} in~$Q_{\forall}$.
% , an accepting state $q_a$, and \emph{looping states}
% $q_\circlearrowleft^1,q_\circlearrowleft^2$
Further, $\Theta$ is the input alphabet and $\Gamma$ is the tape
alphabet that contains a \emph{blank symbol} $\Box \notin \Theta$,
$q_0\in Q_{\exists}$ is the \emph{starting state}, and the
\emph{transition relation} $\Delta$ is of the form
$\Delta\subseteq Q\times \Gamma\times Q\times \Gamma \times \{L,R\}.$
The set $\Delta(q,a):=\{(q',a',M)\mid (q,a,q',a',M)\in\Delta\}$ must
contain exactly two or zero elements for every $q\in Q$ and $a \in
\Gamma$. Moreover, the state $q'$ must be from $Q_\forall$ if $q \in
Q_\exists$ and from $Q_\exists$ otherwise, that is, existential and
universal states alternate. Note that there is no accepting state. The
ATM accepts if it runs forever and rejects otherwise. Starting from
the standard ATM model, this can be achieved by assuming that
exponentially space bounded ATMs terminate on any input and then
modifying them to enter an infinite loop from the accepting state.

A \emph{configuration} of an ATM is a word $wqw'$ with
\mbox{$w,w'\in\Gamma^*$} and $q\in Q$. 
%   A \emph{configuration} $wqw'$ is
%   called \emph{halting} if $q\in\{q_a,q_r\}$ and \emph{accepting} if
%   $q=q_a$. 
We say that $wqw'$ is \emph{existential} if~$q$ is, and likewise for
\emph{universal}.  \emph{Successor configurations} are defined in the
usual way.  Note that every configuration has exactly two successor
configurations.

A \emph{computation tree} of an ATM $M$ on
input $w$ is an infinite tree whose nodes are labeled
with configurations of $M$ such that
\begin{itemize}
	
	\item the root is labeled with the initial configuration
	$q_0w$;
	
	\item if a node is labeled with an existential configuration 
	$wqw'$, then it has a single successor and this successor is labeled 
	with a successor configuration of $wqw'$;
	
	\item if a node is labeled with a universal configuration
	  $wqw'$, then it has two successors and these successors are
	  labeled with the two successor configurations of~$wqw'$.
	
\end{itemize}
An ATM $M$ \emph{accepts} an input $w$ if there is a computation
tree of $M$ on $w$. 

We reduce the word problem for $2^n$-space bounded ATMs which is known
to be \TwoExpTime-hard~\cite{chandraAlternation1981}. The idea of the reduction is as follows. We set
\begin{align*}
\Dmc & =\{A(a),r(b,b),B(b)\}, \\
\Sigma  & = \{r,s,Z,B_{\forall},B_{\exists}^1,B_{\exists}^2\}\cup \{A_\sigma\mid \sigma\in \Gamma\cup (Q\times
\Gamma)\}
\end{align*}
The ontology $\Omc$ enforces that in an $A$-node starts an infinite
$r$-path $\rho$.  Along $\rho$, a counter counts modulo
$2^{n}$ using concept names not in $\Sigma$. In each point of $\rho$
starts an infinite tree along role $s$ that is supposed to mimick the
computation tree of $M$. Along this tree, two counters are maintained:
\begin{itemize}
	
	\item one counter starting at $0$ and counting modulo $2^n$ to
	divide the tree in subpaths of length $2^n$; each such path of
	length $2^n$ represents a configuration;
	
	\item another counter starting at the value of the
	  counter along $\rho$ and
	also counting modulo $2^{n}$.
	
\end{itemize}
To link successive configurations we use
that if $(\Kmc,\{a\},\{b\})$ has no strong $\ALCI(\Sigma)$
solution, then there exist models \Amf and \Bmf of $\Kmc$ such
that $\Amf,a^\Amf\sim_{\Sigma} \Bmf,b^\Bmf$: from $r(b,b)\in \Dmc$ it follows
that in \Amf all nodes on the $r$-path $\rho$ are $\Sigma$-bisimilar.
Thus, each node on the $\rho$ is the starting point of $s$-trees
with identical $\Sigma$-decorations. As on the $m$th $s$-tree the
second counter starts at all nodes at distances $k\times 2^{n}-m$,
for all $k\geq 1$, we are in the position to coordinate all positions
at all successive configurations.

The ontology $\Omc$ is constructed as follows. We first enforce the
infinite $r$-path $\rho$ with the counter, which is realized using
concept names $A_i$, $\overline A_i$, $i<n$:
\begin{align*}
A & \sqsubseteq I_s \sqcap \bigsqcap_{i<n} \overline{A}_i
\\
I_s& \sqsubseteq \exists r.\top\sqcap \forall r.I_s\\
A_i \sqcap \bigsqcap_{j < i} A_j & \sqsubseteq \forall r.
\overline{A}_i\\
\overline{A}_i \sqcap \bigsqcap_{j < i} A_j & \sqsubseteq \forall r.
A_i\\
A_i \sqcap \bigsqcup_{j < i} \overline{A}_j & \sqsubseteq \forall r.
A_i\\
\overline{A}_i \sqcap \bigsqcup_{j < i} \overline{A}_j & \sqsubseteq
\forall r. \overline{A}_i
\end{align*}
Note that all points of the $r$-path satisfy a concept name $I_s$,
from which we start the $s$-trees with two counters, realized using
concept names $U_i,\overline U_i$ and $V_i,\overline V_i$, $i<n$, and
initialized to $0$ and the value of the $A$-counter, respectively:
\begin{align*}
I_s & \sqsubseteq (U=0)\\
I_s \sqcap A_j & \sqsubseteq V_j &j<n\\
I_s \sqcap \overline A_j & \sqsubseteq \overline V_j &j<n \\
\top & \sqsubseteq \exists s.\top
\end{align*}
Here, $(U=0)$ is an abbreviation for the concept
$\bigsqcap_{i=1}^n\overline U_i$, we use similar abbreviations below.
The counters $U_i$ and $V_i$ are incremented along $s$ analogously to
how $A_i$ is incremented along $r$, so we omit details. Configurations
of $M$ are represented between two consecutive points having
$U$-counter value $0$. We next enforce the structure of the
computation tree, assuming that $q_0\in Q_\forall$:
\begin{align*}
  I_s & \sqsubseteq B_{\forall} \\
  (U<2^n-1) \sqcap B_\forall & \sqsubseteq \forall s.B_{\forall} \\
  (U<2^n-1) \sqcap B_\exists^i & \sqsubseteq \forall
  s.B_{\exists}^i && i\in\{1,2\} \\
  (U=2^n-1) \sqcap B_\forall & \sqsubseteq \forall
  s.(B_{\exists}^1\sqcup B_{\exists}^2) \\
  (U=2^n-1) \sqcap (B_\exists^1\sqcup B_{\exists}^2) & \sqsubseteq \forall
  s.B_\forall \\
  (U=2^n-1) \sqcap B_{\forall} & \sqsubseteq \exists s.Z\sqcap \exists
  s.\neg Z
   %
%   (B_{\exists}^1\sqcap B_{\exists}^2) \sqcup (B_{\exists}^1\sqcap
%   B_{\forall}) \sqcup (B_{\exists}^2\sqcap
%   B_{\forall}) & \sqsubseteq \bot
\end{align*}
These sentences enforce that all points which represent a configuration
satisfy exactly one of $B_{\forall},B_{\exists}^1,B_{\exists}^2$
indicating the kind of configuration and, if existential, also a choice of the transition function. The symbol $Z\in \Sigma$ enforces the branching.

We next set the initial configuration, for input $w=a_0,\dots,a_{n-1}$.
\begin{align*}
A & \sqsubseteq  A_{q_{0},a_0}\\
A & \sqsubseteq  \forall s^{k}.A_{a_{k}} & 0<k<n \\
A & \sqsubseteq  \forall s^{n+1}.\mn{Blank} \\
\mn{Blank} & \sqsubseteq A_{\Box} \\
\mn{Blank} \sqcap (U<2^n -1) & \sqsubseteq \forall s. \mn{Blank}
\end{align*}
To coordinate consecutive configurations, we associate with $M$
functions $f_i$, $i\in \{1,2\}$ that map the content of three
consecutive cells of a configuration to the content of the middle cell
in the $i$-the successor configuration (assuming an arbitrary order on
the set $\Delta(q,a)$). In what follows, we ignore the cornercasees that occur
at the border of configurations; they can be treated in a similar way. 
Clearly, for each possible such triple
$(\sigma_1,\sigma_2,\sigma_3)\in \Gamma\cup(Q\times \Gamma)$, there is
an \ALC concept $C_{\sigma_1,\sigma_2,\sigma_3}$ which is true at an
element $a$ of the computation tree iff $a$ is labeled with
$A_{\sigma_1}$, $a$'s $s$-successor $b$ is labeled with
$A_{\sigma_2}$, and $b$'s $s$-successor $c$ is labeled with
$A_{\sigma_3}$. Now, in each configuration, we synchronize elements
with $V$-counter $0$ by including for every $\vec\sigma =
(\sigma_1,\sigma_2,\sigma_3)$ and $i\in\{1,2\}$ the following
sentences: 
\begin{align*}
  (V=2^n-1) \sqcap (U<2^n-2)\sqcap C_{\sigma_1,\sigma_2,\sigma_3} & \sqsubseteq
  \forall s. A^1_{f_1(\sigma)}\sqcap \forall s.A^2_{f_2(\sigma)} \\
  (V=2^n-1) \sqcap (U<2^n-2)\sqcap C_{\sigma_1,\sigma_2,\sigma_3} \sqcap B_\exists^i&
  \sqsubseteq \forall s. A^i_{f_i(\sigma)}
 \end{align*}
The concept names $A^i_{\sigma}$ are used as markers (not in $\Sigma$) and
are propagated along $s$ for $2^n$ steps, exploiting the $V$-counter.
The superscript $i\in\{1,2\}$ determines the successor configuration
that the symbol is referring to. After crossing the end of a
configuration, the symbol $\sigma$ is propagated using concept names
$A_{\sigma}'$ (the superscript is not needed anymore because the
branching happens at the end of the configuration, based on $Z$).
\begin{align*}
  (U<2^n-1) \sqcap A_\sigma^i & \sqsubseteq \forall s. A_{\sigma}^i \\
  (U=2^n-1) \sqcap B_\forall \sqcap A_\sigma^1 & \sqsubseteq \forall
  s.(\neg Z\sqcup A'_\sigma)\\
  (U=2^n-1) \sqcap B_\forall \sqcap A_\sigma^2 & \sqsubseteq \forall
  s.(Z\sqcup A'_\sigma)\\
  (U=2^n-1) \sqcap B_\exists^i \sqcap A_\sigma^i & \sqsubseteq \forall
  s.A'_\sigma && i\in\{1,2\}\\
  (V<2^n-1) \sqcap A'_\sigma & \sqsubseteq \forall s.A'_\sigma \\
  (V=2^n-1) \sqcap A'_\sigma & \sqsubseteq \forall s.A_\sigma
\end{align*}
For those $(q,a)$ with $\Delta(q,a)=\emptyset$, we add the concept inclusion
\[A_{q,a} \sqsubseteq \bot. \]
The following Claim establishes correctness of the reduction

\medskip\noindent\textit{Claim.} 
$M$ accepts the input $w$ iff there exist models $\Amf,\Bmf$ of
$\Kmc$ such that $\Amf,a^\Amf\sim_{\Sigma}\Bmf,b^\Bmf$.

\medskip\noindent\textit{Proof of the Claim.} $(\Rightarrow)$ If $M$
accepts $w$, there is a computation tree of $M$ on $w$. We construct a
single interpretation $\Amf$ as follows. Let $\Amf^*$ be the infinite
tree-shaped structure that represents the computation tree of $M$ on
$w$ as described above, that is, configurations are represented by
sequences of $2^n$ elements linked by role $s$ and labeled by
$B_\forall,B_\exists^1,B_\exists^2$ depending on whether the
configuration is universal or existential, and in the latter case the
superscript indicates which choice has been made for the existential
state. Finally, the first element of the first successor configuration
of a universal configuration is labeled with $Z$. Observe that 
$\Amf^*$ interprets only the symbols in $\Sigma$ as non-empty. Now, we
obtain structures $\Amf_k$, $k<2^n$ from $\Amf^*$ by interpreting
non-$\Sigma$-symbols as follows: 
\begin{itemize}
	
  \item the root of $\Amf_k$ satisfies $I_s$; 

  \item the $U$-counter starts at $0$ at the root and counts modulo
    $2^n$ along each $s$-path;

  \item the $V$-counter starts at $k$ at the root and counts modulo
    $2^n$ along each $s$-path;

  \item the auxiliary concept names of the shape $A_\sigma^i$ and
    $A_\sigma'$ are interpreted in a minimal way so as to satisfy the
    concept inclusions  listed above. Note that the respective concept
    inclusions are Horn, hence there is no choice.

\end{itemize}
Now obtain $\Amf$ from $\Amf^*$ and the $\Amf_k$ by creating an (both
side) infinite $r$-path $\rho$ through $a^\Amf=a$ (with the
corresponding $A$-counter) and adding all $\Amf_k$ to every node on
the $r$-path by identifying the roots of the $\Amf_k$ with the node on
the path.  Additionally, add $\Amf^*$ to $b^\Amf=b$ by identifying $b$
with the root of $\Amf^*$. It should be clear that $\Amf$ is as
required. In particular, \Amf is a model of $\Omc$ and the reflexive
and symmetric closure of 
\begin{itemize}

  \item all pairs $(b,e), (e,e')$, with $e,e'$ on $\rho$, and 

  \item all pairs $(e,e'), (e',e'')$, with $e$ in $\Amf^*$ and $e',e''$
    copies of $e$ in the trees $\Amf_{k}$.

\end{itemize}
is an $\ALCI(\Sigma)$-bisimulation $S$ on $\Amf$ with $(b,a)\in S$.

$(\Leftarrow)$ Let $\Amf,\Bmf$ be models of \Kmc such that
$\Amf,a^\Amf\sim_{\Sigma}\Bmf,b^\Bmf$. As it was argued above, due to
the $r$-self loop at $b^\Amf$, from $a^\Amf$ there has to be an
outgoing infinite $r$-path on which all $s$-trees are
$\Sigma$-bisimilar. There is also an outgoing infinite $r^-$-path with
this property, but it is not relevant for the proof.  All those
$s$-trees are additionally labeled with some auxiliary concept names
not in $\Sigma$, depending on the distance from $a^\Amf$. However, it
can be shown using the CIs in \Omc that all $s$-trees contain a
computation tree of $M$ on input $w$.

\bigskip
Note that we have to take care of inverses in the correctness proof
since the characterization refers to \ALCI-bisimulations. Since the
ontology is actually an \ALC-ontology (and there is a similar
characterization), also strong separability in \ALC is
\TwoExpTime-hard.

\section{Strong Separability in GF and GNF}
We show that strong separability with signature is decidable in the
guarded fragment, GF, and the guarded negation fragment, GNF, of FO.
We also obtain a \ThreeExpTime upper bound for GF and
\TwoExpTime-completeness for GNF. Finally, we show that strong GNF separability with signature restrictions coincides with strong FO separability with signature restrictions for labeled GNF-KBs. The analogous result does not hold for GF. The proofs are based on the link to interpolants and the CIP discussed in Section~\ref{sec:dfstrong}.

The formulas $\varphi_{\Sigma,\vec{a}}(\vec{x})$ and $\neg\varphi_{\Sigma,\vec{b}}(\vec{x})$ defined in Section~\ref{sec:dfstrong} 
are not in GF nor GNF, even if $\Kmc$ is a KB in GF or, respectively,
GNF. To obtain formulas in GF and GNF, take
fresh relation symbols $R_{\Dmc,\vec{a}}$ and $R_{\Dmc,\vec{b}}$
of arity $n$, where $n$ is the number of constants
in $\Dmc$. Then add $R_{\Dmc,\vec{a}}(\vec{y})$ to
$\Kmc_{\Sigma,\vec{a}}$ when constructing
$\varphi_{\Sigma,\vec{a}}(\vec{x})$, where $\vec{y}$ is an enumeration 
of the variables in $\Kmc_{\Sigma,\vec{a}}$. Denote the resulting
formula by $\varphi_{\Sigma,\vec{a}}'(\vec{x})$.
Do the same to construct $\varphi_{\Sigma,\vec{b}}'(\vec{x})$,
using $R_{\Dmc,\vec{b}}$ instead of $R_{\Dmc,\vec{a}}$.
The formulas $\varphi_{\Sigma,\vec{a}}'$ and $\neg\varphi_{\Sigma,\vec{b}}'$ are in GF and GNF if the
KB is given in GF and, respectively, GNF. By construction we
obtain the following result.
\begin{theorem}~\label{thm:red-interpolation}
	Let $\Lmc\in \{\text{GF},\text{GNF}\}$. Then there is a polynomial
	time reduction of strong $\Lmc$-separability with signature to
	$\Lmc$-interpolant existence. Moreover, given a labeled $\LmcO$-KB $(\Kmc,\{\vec{a}\},\{\vec{b}\})$ and $\Sigma\subseteq \text{sig}(\Kmc)$,  the following conditions are equivalent for any formula $\varphi$ in $\Lmc$:
	\begin{enumerate}
		\item $\varphi$ strongly $\Lmc(\Sigma)$-separates $(\Kmc,\{\vec{a}\},\{\vec{b}\})$;
		\item $\varphi$ is an $\Lmc$-interpolant for $\varphi_{\Sigma,\vec{a}}'(\vec x),\neg \varphi_{\Sigma,\vec{b}}'(\vec{x})$.
	\end{enumerate}	
\end{theorem}

It has been proved
in~\cite{DBLP:journals/tocl/BenediktCB16,DBLP:journals/jsyml/BaranyBC18}
that GNF has the CIP. Thus, we obtain the following result.
\begin{theorem}
  Strong GNF-separability with signature is \TwoExpTime-complete.
  Moreover, a GNF-KB $(\Kmc, \{\vec{a}\},\{\vec{b}\})$ is strongly
  GNF($\Sigma$)-separable iff it is strongly FO$(\Sigma)$-separable.
\end{theorem}
In contrast, GF does not enjoy the CIP~\cite{DBLP:journals/sLogica/HooglandM02} and so interpolant
existence in GF does not reduce to a validity. In fact,
decidability and \ThreeExpTime-completeness for GF-interpolant
existence has only recently been established~\cite{jung2020living}.
From this result and the reduction in
Theorem~\ref{thm:red-interpolation}, we obtain a \ThreeExpTime-upper
bound for strong GF-separability with signature. A matching lower
bound can be shown similar to the lower bound for GF-interpolant
existence. 
\begin{theorem}
  Strong GF-separability with signature is \ThreeExpTime-complete.
\end{theorem}
As GF does not enjoy the CIP, we also do not obtain that
a GF-KB $(\Kmc, \{\vec{a}\},\{\vec{b}\})$ is strongly GF($\Sigma$)-separable iff it is strongly FO$(\Sigma)$-separable.
In fact, the counterexample to CIP constructed in~\cite{DBLP:journals/jsyml/BaranyBC18} is easily
adapted to show the following.
\begin{theorem}
	Strong FO($\Sigma)$-separability of a 
	GF-KB $(\Kmc, \{\vec{a}\},\{\vec{b}\})$ does not imply strong GF($\Sigma$)-separabilty of $(\Kmc, \{\vec{a}\},\{\vec{b}\})$.
\end{theorem}

\section{Conclusion}
We have investigated the complexity of deciding weak and strong separability of labeled KBs with signature restrictions for $\mathcal{ALCI}$ and guarded fragments of FO, and observed a close link between weak separability and uniform interpolants on the one hand, and between strong separability and Craig interpolants on the other. Numerous questions remain to be explored: what is the size of separating formulas and how can they be computed efficiently, if they exist? What is the complexity of weak non-projective separability with signature restrictions for $\mathcal{ALCI}$? We conjecture that this is still \TwoExpTime-complete but lack a proof. What happens for DLs with number restrictions and/or nominals? We have shown that weak projective separability is undecidable for $\mathcal{ALCFIO}$ with signature restrictions, but it could well be decidable for $\mathcal{ALCQO}$. For strong separability, there are many exciting open problems: is strong separability with signature restrictions decidable for $\mathcal{ALCFIO}$? In this case, even the case without signature restrictions has not yet been investigated and could well already be tricky. Is it decidable for the two-variable fragment of FO? For the two-variable fragment, the case without signature restrictions has been investigated in~\cite{KR}, and \NExpTime-completeness established.
Attacking these problems is closely related to deciding the existence of Craig interpolants and computing (good) separating formulas is closely related to computing (good) Craig interpolants. Also of interest are
the same questions for Horn DLs. The situation for $\mathcal{EL}$ and $\mathcal{ELI}$ has been explored in \cite{DBLP:conf/ijcai/FunkJLPW19,aaaithis}, but more expressive ones have not yet been considered. 

\bibliographystyle{splncs03}
\bibliography{local}

\begin{thebibliography}{10}
\providecommand{\url}[1]{\texttt{#1}}
\providecommand{\urlprefix}{URL }

\bibitem{ANvB98}
Andr{\'{e}}ka, H., N{\'{e}}meti, I., {van Benthem}, J.: Modal languages and
  bounded fragments of predicate logic. J. Philosophical Logic  27(3),
  217--274 (1998)

\bibitem{DBLP:journals/tods/ArenasD16}
Arenas, M., Diaz, G.I.: The exact complexity of the first-order logic
  definability problem. {ACM} Trans. Database Syst.  41(2),  13:1--13:14 (2016)

\bibitem{DBLP:conf/www/ArenasDK16}
Arenas, M., Diaz, G.I., Kostylev, E.V.: Reverse engineering {SPARQL} queries.
  In: Proc.\ of {WWW}. pp. 239--249 (2016)

\bibitem{All2020}
Artale, A., Jung, J.C., Mazzullo, A., Ozaki, A., Wolter, F.: Living without
  {B}eth and {C}raig: Explicit definitions and interpolants in description
  logics with nominals (2020), submitted

\bibitem{handbook}
Baader, F., Deborah, Calvanese, D., McGuiness, D.L., Nardi, D.,
  Patel-Schneider, P.F. (eds.): The Description Logic Handbook. Cambridge
  University Press (2003)

\bibitem{DL-Textbook}
Baader, F., Horrocks, I., Lutz, C., Sattler, U.: An Introduction to Description
  Logics. Cambride University Press (2017)

\bibitem{DBLP:conf/ilp/BadeaN00}
Badea, L., Nienhuys{-}Cheng, S.: A refinement operator for description logics.
  In: Proc. of {ILP}. pp. 40--59 (2000)

\bibitem{DBLP:journals/jsyml/BaranyBC18}
B{\'{a}}r{\'{a}}ny, V., Benedikt, M., ten Cate, B.: Some model theory of
  guarded negation. J. Symb. Log.  83(4),  1307--1344 (2018)

\bibitem{DBLP:conf/icdt/Barcelo017}
Barcel{\'{o}}, P., Romero, M.: The complexity of reverse engineering problems
  for conjunctive queries. In: Proc.\ of {ICDT}. pp. 7:1--7:17 (2017)

\bibitem{DBLP:journals/tocl/BenediktCB16}
Benedikt, M., ten Cate, B., {Vanden Boom}, M.: Effective interpolation and
  preservation in guarded logics. {ACM} Trans. Comput. Log.  17(2),  8:1--8:46
  (2016)

\bibitem{DBLP:conf/kr/BorgidaTW16}
Borgida, A., Toman, D., Weddell, G.E.: On referring expressions in query
  answering over first order knowledge bases. In: Proc. of {KR}. pp. 319--328
  (2016)

\bibitem{DBLP:journals/ai/BotoevaKRWZ16}
Botoeva, E., Kontchakov, R., Ryzhikov, V., Wolter, F., Zakharyaschev, M.: Games
  for query inseparability of description logic knowledge bases. Artif. Intell.
   234,  78--119 (2016)

\bibitem{DBLP:journals/ai/BotoevaLRWZ19}
Botoeva, E., Lutz, C., Ryzhikov, V., Wolter, F., Zakharyaschev, M.: Query
  inseparability for {ALC} ontologies. Artif. Intell.  272,  1--51 (2019)

\bibitem{DBLP:conf/www/BuhmannLWB18}
B{\"{u}}hmann, L., Lehmann, J., Westphal, P., Bin, S.: {DL}-learner -
  structured machine learning on semantic web data. In: Proc. of {WWW}. pp.
  467--471 (2018)

\bibitem{TenEtAl13}
ten Cate, B., Franconi, E., Seylan, I.: Beth definability in expressive
  description logics. J. Artif. Intell. Res.  48,  347--414 (2013)

\bibitem{chandraAlternation1981}
Chandra, A.K., Kozen, D.C., Stockmeyer, L.J.: Alternation. J. ACM  28,
  114--133 (1981)

\bibitem{DBLP:conf/ekaw/Fanizzi0dE18}
Fanizzi, N., Rizzo, G., d'Amato, C., Esposito, F.: {DLFoil}: Class expression
  learning revisited. In: Proc. of {EKAW}. pp. 98--113 (2018)

\bibitem{DBLP:conf/ijcai/FunkJLPW19}
Funk, M., Jung, J.C., Lutz, C., Pulcini, H., Wolter, F.: Learning description
  logic concepts: When can positive and negative examples be separated? In:
  Proc. of {IJCAI}. pp. 1682--1688 (2019)

\bibitem{DBLP:conf/kr/GhilardiLW06}
Ghilardi, S., Lutz, C., Wolter, F.: Did {I} damage my ontology? {A} case for
  conservative extensions in description logics. In: Proc.\ of {KR}. pp.
  187--197. {AAAI} Press (2006)

\bibitem{goranko20075}
Goranko, V., Otto, M.: Model theory of modal logic. In: Handbook of Modal
  Logic, pp. 249--329. Elsevier (2007)

\bibitem{DBLP:journals/jsyml/Gradel99}
Gr{\"{a}}del, E.: On the restraining power of guards. J. Symb. Log.  64(4),
  1719--1742 (1999)

\bibitem{HKGrauS08}
Grau, B.C., Horrocks, I., Kazakov, Y., Sattler, U.: Modular reuse of
  ontologies: Theory and practice. J.\ of Artifical Intelligence Research  31,
  273--318 (2008)

\bibitem{GuJuSa-IJCAI18}
Guti{\'e}rrez-Basulto, V., Jung, J.C., Sabellek, L.: Reverse engineering
  queries in ontology-enriched systems: The case of expressive {Horn}
  description logic ontologies. In: Proc. of {IJCAI-ECAI} (2018)

\bibitem{DBLP:journals/sLogica/HooglandM02}
Hoogland, E., Marx, M.: Interpolation and definability in guarded fragments.
  Studia Logica  70(3),  373--409 (2002)

\bibitem{JLMSW17}
Jung, J., Lutz, C., Martel, M., Schneider, T., Wolter, F.: Conservative
  extensions in guarded and two-variable fragments. In: Proc.\ of {ICALP}. pp.
  108:1--108:14. Schloss Dagstuhl -- LZI (2017)

\bibitem{KR}
Jung, J.C., Lutz, C., Pulcini, H., Wolter, F.: Logical separability of
  incomplete data under ontologies. In: Proc.\ of {KR}. IJCAI (2020)

\bibitem{aaaithis}
Jung, J.C., Lutz, C., Wolter, F.: Least general generalizations in description
  logic: Verification and existence. In: Proc. of {AAAI}. pp. 2854--2861.
  {AAAI} Press (2020)

\bibitem{jung2020living}
Jung, J.C., Wolter, F.: Living without beth and craig: Explicit definitions and
  interpolants in the guarded fragment (2020), available at
  http://arxiv.org/abs/2007.01597

\bibitem{kalashnikov2018fastqre}
Kalashnikov, D.V., Lakshmanan, L.V., Srivastava, D.: Fastqre: Fast query
  reverse engineering. In: Proc.\ of {SIGMOD}. pp. 337--350 (2018)

\bibitem{DBLP:journals/tods/KimelfeldR18}
Kimelfeld, B., R{\'{e}}, C.: A relational framework for classifier engineering.
  {ACM} Trans. Database Syst.  43(3),  11:1--11:36 (2018),
  \url{https://doi.org/10.1145/3268931}

\bibitem{KonevLWW09}
Konev, B., Lutz, C., Walther, D., Wolter, F.: Formal properties of
  modularisation. In: Modular Ontologies, Lecture Notes in Computer Science,
  vol. 5445, pp. 25--66. Springer (2009)

\bibitem{DBLP:journals/coling/KrahmerD12}
Krahmer, E., van Deemter, K.: Computational generation of referring
  expressions: {A} survey. Computational Linguistics  38(1),  173--218 (2012)

\bibitem{polconcept_learning}
Lehmann, J., Fanizzi, N., B{\"u}hmann, L., d'Amato, C.: Concept learning. In:
  Perspectives on Ontology Learning, pp. 71--91. AKA / IOS Press (2014)

\bibitem{DBLP:journals/ml/LehmannH10}
Lehmann, J., Hitzler, P.: Concept learning in description logics using
  refinement operators. Machine Learning  78,  203--250 (2010)

\bibitem{TBoxpaper}
Lutz, C., Piro, R., Wolter, F.: Description logic {T}{B}oxes: Model-theoretic
  characterizations and rewritability. In: Proc. of {IJCAI} (2011)

\bibitem{DBLP:conf/ijcai/LutzWW07}
Lutz, C., Walther, D., Wolter, F.: Conservative extensions in expressive
  description logics. In: Proc.\ of {IJCAI}. pp. 453--458 (2007)

\bibitem{DBLP:conf/ijcai/LutzW11}
Lutz, C., Wolter, F.: Foundations for uniform interpolation and forgetting in
  expressive description logics. In: Proc.\ of {IJCAI}. pp. 989--995.
  {IJCAI/AAAI} (2011)

\bibitem{martins2019reverse}
Martins, D.M.L.: Reverse engineering database queries from examples:
  State-of-the-art, challenges, and research opportunities. Information Systems
   (2019)

\bibitem{DBLP:conf/gcai/Ortiz19}
Ortiz, M.: Ontology-mediated queries from examples: a glimpse at the {DL-Lite}
  case. In: Proc. of {GCAI}. pp. 1--14 (2019)

\bibitem{DBLP:conf/semweb/PetrovaKGH19}
Petrova, A., Kostylev, E.V., Grau, B.C., Horrocks, I.: Query-based entity
  comparison in knowledge graphs revisited. In: Proc.\ of {ISWC}. pp. 558--575.
  Springer (2019)

\bibitem{DBLP:conf/semweb/PetrovaSGH17}
Petrova, A., Sherkhonov, E., Grau, B.C., Horrocks, I.: Entity comparison in
  {RDF} graphs. In: Proc.\ of {ISWC}. pp. 526--541 (2017)

\bibitem{DBLP:conf/aaai/SarkerH19}
Sarker, M.K., Hitzler, P.: Efficient concept induction for description logics.
  In: Proc.\ of {AAAI}. pp. 3036--3043 (2019)

\bibitem{DBLP:conf/sigmod/TranCP09}
Tran, Q.T., Chan, C., Parthasarathy, S.: Query by output. In: Proc.\ of {PODS}.
  pp. 535--548. {ACM} (2009)

\bibitem{DBLP:journals/vldb/TranCP14}
Tran, Q.T., Chan, C.Y., Parthasarathy, S.: Query reverse engineering. {VLDB} J.
   23(5),  721--746 (2014)

\bibitem{DBLP:journals/fuin/TranHHNN14}
Tran, T., Ha, Q., Hoang, T., Nguyen, L.A., Nguyen, H.S.: Bisimulation-based
  concept learning in description logics. Fundam. Inform.  133(2-3),  287--303
  (2014)

\bibitem{DBLP:conf/icalp/Vardi98}
Vardi, M.Y.: Reasoning about the past with two-way automata. In: Proc. of
  ICALP'98. pp. 628--641 (1998)

\bibitem{DBLP:conf/pods/WeissC17}
Weiss, Y.Y., Cohen, S.: Reverse engineering spj-queries from examples. In:
  Proc.\ of {PODS}. pp. 151--166. {ACM} (2017)

\bibitem{DBLP:conf/sigmod/ZhangEPS13}
Zhang, M., Elmeleegy, H., Procopiuc, C.M., Srivastava, D.: Reverse engineering
  complex join queries. In: Proc.\ of {SIGMOD}. pp. 809--820. {ACM} (2013)

\end{thebibliography}

\newpage
%\end{document}

\appendix

\section{Proof of Theorem~\ref{thm:alci-refinement}}

We formulate the result to be shown again.

\medskip
\noindent
{\bf Theorem~\ref{thm:alci-refinement}}
	Assume a labeled $\ALCI$-KB $(\Kmc,P,\{b\})$ and $\Sigma\subseteq \text{sig}(\Kmc)$ are given.
	Then the following conditions are equivalent:
	\begin{enumerate}
		\item $\Kmc=(\Omc,\Dmc)$ is projectively $\mathcal{ALCI}(\Sigma)$-separable.
		\item there exists a forest model $\Amf$ of $\Kmc$ of finite outdegree and a signature $\Sigma'$ such that  $\Sigma' \cap \text{sig}(\Kmc)\subseteq \Sigma$
		and for all models $\Bmf$ of $\Kmc$ and all $a\in P$:
		$
		\Bmf,a^{\Bmf} \not\sim_{\ALCI,\Sigma'} \Amf, b^{\Amf}.
		$
		\item there exists a forest model $\Amf$ of $\Kmc$ of finite outdegree 
		such that for all models $\Bmf$ of $\Kmc$ and all $a\in P$:
		$
		\Bmf,a^{\Bmf} \not\sim_{\ALCI,\Sigma}^{f} \Amf, b^{\Amf}.
		$
		\item there exists a forest model $\Amf$ of $\Kmc$ of finite outdegree such that for all $a\in P$: $\Dmc_{\text{con}(a)},a \not\rightarrow^{\Sigma}_{c} \Amf,b^{\Amf}$.
	\end{enumerate}  

\medskip
\noindent
\begin{proof}
``1 $\Rightarrow$ 2''. Take an $\ALCI$-concept $C$ with $\text{sig}(C) \cap \text{sig}(\Kmc) \subseteq \Sigma$ such $C$ separates
$(\Kmc,P,\{b\})$. There exists a model $\Amf$ of $\Kmc$ of finite outdegree
such that $b^{\Amf}\in (\neg C)^{\Amf}$. Let $\Sigma'= \text{sig}(C)$. Then $\Amf$
and $\Sigma'$ are as required for Condition~2.

``2 $\Rightarrow$ 3''. Take a forest model $\Amf$ and $\Sigma'$ such that
Condition 2 holds. 
%We may assume that for any two distinct $d,d'\in \text{dom}(\Amf)$ there exist %distinct $A_{d},A_{d'}\in \Sigma'$ with 
%$d\in A_{d}^{\Amf}\setminus A_{d'}^{\Amf}$ and $d'\in A_{d'}^{\Amf}\setminus %A_{d}^{\Amf}$, respectively. 
We claim that Condition~3 holds for $\Amf$ as well. 
Suppose that there exists a model $\Bmf$ of $\Kmc$, $a\in P$, and a functional $\Sigma$-bisimulation $f$ witnessing $\Bmf,a^{\Bmf} \sim_{\ALCI,\Sigma}^{f} \Amf,b^{\Amf}$.

Define $\Bmf'$ by expanding $\Bmf$ as follows: 
\begin{itemize}
	\item for every concept name $A\in \Sigma'\setminus \text{sig}(\Kmc)$
	and $d\in \text{dom}(f)$, let $d\in A^{\Bmf'}$ if $f(d) \in A^{\Amf}$;
	\item for every role $R$ over $\Sigma'\setminus \text{sig}(\Kmc)$
	and $d\in \text{dom}(f)$, if there exists $e\in \text{dom}(\Amf)$ with $(f(d),e)\in R^{\Amf}$,
	  then add a disjoint copy of $\Amf$ to $\Bmf$ and add $(d,e')$ to $R^{\Bmf'}$ for the copy $e'$ of $e$.  
\end{itemize}
It is easy to see that $\Bmf',a^{\Bmf'} \sim_{\ALCI,\Sigma'} \Amf,b^{\Amf}$, and we have derived a contradiction.

``3 $\Rightarrow$ 4''. Take a forest model $\Amf$ such that
Condition 3 holds. We claim that Condition~4 holds for $\Amf$ as well. 
For a proof by contradiction let $h$ be a $\Sigma$-homomorphism 
and $t_{d}$, $d\in \text{dom}(\Dmc)$, be $\Kmc$-types,
and $a\in P$ such $h$ refutes Condition 4. Take models $\Bmf_{d}$ of $\Omc$
such that $\Bmf_{d},d \sim_{\ALCI,\Sigma} \Amf,h(d)$. We may assume that the $\Bmf_{d}$ are tree-shaped with root $d$ and that the bisimulations are functions $f_{d}$.
Now attach to every $d\in \text{dom}(\Dmc)$ the model $\Bmf_{d}$ and obtain
$\Bmf$ by adding $(d,d')$ to $R^{\Bmf}$ if $R(d,d')\in \Dmc$. Then
$$
f= \bigcup_{d\in \text{dom}(\Dmc)}f_{d}
$$
is a functional $\ALCI(\Sigma)$-bisimulation between $\Bmf$ and $\Amf$.

``4 $\Rightarrow$ 3''. Take a forest model $\Amf$ such that
Condition 4 holds. We claim that Condition~3 holds for $\Amf$ as well. 
For a proof by contradiction let $f$ be a functional $\ALCI(\Sigma)$-bisimulation witnessing $\Bmf,a^{\Bmf}\sim_{\ALCI,\Sigma}^{f} \Amf,b^{\Amf}$ for 
some model $\Bmf$ of $\Kmc$.
The restriction $h$ of $f$ of $\Dmc$ is the $\Sigma$-homomorphism 
needed to refute Condition 4.
 
``3 $\Rightarrow$ 2''. Take a model $\Amf$ such that
Condition 3 holds. 
Define $\Amf'$ by expanding $\Amf$ as follows. Take for any $d\in \text{dom}(\Amf)$ a fresh concept name $A_{d}$ and set
$A_{d}^{\Amf'} =\{d\}$. 
Clearly Condition~2 holds for $\Amf'$ and $\Sigma'=\Sigma \cup \{A_{d}\mid d\in \text{dom}(\Amf)\}$. 

``2 $\Rightarrow$ 1''. Straightforward.
\end{proof}

\subsection{Additional Definitions for 2ATAs}

We make precise the semantics of 2ATAs. Let
$\Amc=(Q,\Theta,q_0,\delta,\Omega)$ be a 2ATA and $(T,L)$ a
$\Theta$-labeled tree. A {\em run for \Amc on $(T,L)$} is a $T\times
Q$-labeled tree $(T_r,r)$ such that:
\begin{itemize}

  \item $\varepsilon\in T_r$ and $r(\varepsilon)=(\varepsilon,q_0)$;

  \item For all $y\in T_r$ with $r(y)=(x,q)$ and $\delta(q,L(x))=\vp$,
    there is an assignment $v$ of truth values to the transitions in $\vp$
    such that $v$ satisfies $\vp$ and:
    \begin{itemize}

      \item if $v(p)=1$, then $r(y')=(x,p)$ for some successor
	$y'$ of $y$ in $T_r$;

      \item if $v(\langle - \rangle p)=1$, then $x \neq \varepsilon$
        and there is a successor $y'$ of $y$ in $T_r$ with
        $r(y')=(x\cdot -1,p)$;

      \item if $v([-] p)=1$, then $x=\varepsilon$ or there is a
        successor $y'$ of $y$ in $T_r$ such that 
	$r(y')=(x\cdot -1,p)$;
        
      \item if $v(\Diamond p)=1$, then there is a successor $x'$ of
        $x$ in $T$ and a successor $y'$ of $y$ in $T_r$ such that
        $r(y')=(x',p)$;

      \item if $v(\Box p)=1$, then for every successor $x'$ of
        $x$ in $T$, there is a successor $y'$ of $y$ in $T_r$ such that
        $r(y')=(x',p)$.

    \end{itemize}

\end{itemize}
Let $\gamma=i_0i_1\cdots$ be an infinite path in $T_r$ and denote, for
all $j\geq 0$, with $q_j$ the state such that $r(i_0\cdots
i_j)=(x,q_j)$. The path $\gamma$ is {\em accepting} if the largest
number $m$ such that $\Omega(q_j)=m$ for infinitely many $j$ is even.
A run $(T_r,r)$ is accepting, if all infinite paths in $T_r$ are
accepting. Finally, a tree is accepted if there is some accepting run
for it.

\end{document}